\documentclass[twoside,11pt]{article}
\usepackage{jmlr2em}
\usepackage{amsmath,color}

\usepackage{enumitem}

\def\ga{\alpha}

\def\gb{\beta}

\def\gd{\delta}
\def\gD{\Delta}
\def\gep{\varepsilon}

\def\gO{\Omega}

\def\R{\mathbb{R}}

\def\X{\mathcal{X}}
\def\Y{\mathcal{Y}}
\def\Z{\mathcal{Z}}

\def\D{\mathcal{D}}
\def\L{{\mathcal{L}}}

\def\EX{{\mathbb{E}}}

\def\bu{\mathbf{u}}

\def\R{\mathbb{R}}

\def\bw{\mathbf{w}}
\def\hbw{{\widetilde{\mathbf{w}}_n}}

\def\ASL{A^{\text{last}}(S)}
\def\SLST{\mathcal{E}_{\text{stab}}^{last}\,}
\def\SAVG{\mathcal{E}_{\text{stab}}^{avg}\,}
\def\OLST{\mathcal{E}_{\text{opt}}^{last}\,}
\def\OAVG{\mathcal{E}_{\text{opt}}^{avg}\,}

\def\ASA{A^{\text{avg}}(S)}

\def\begth{\begin{theorem}}
\def\endth{\end{theorem}}
\def\begprop{\begin{proposition}}
\def\endprop{\end{proposition}}
\def\begcor{\begin{corollary}}
\def\endcor{\end{corollary}}
\def\begdef{\begin{definition}}
\def\enddef{\end{definition}}
\def\beglemm{\begin{lemma}}
\def\endlemm{\end{lemma}}
\def\begexm{\begin{example}}
\def\endexm{\end{example}}
\def\begrem{\begin{remark}}
\def\endrem{\end{remark}}

\newcommand\numberthis{\addtocounter{equation}{1}\tag{\theequation}}

\newcommand*{\dom}{\mathrm{dom}}

\newcommand*{\argmin}{\mathrm{argmin}}

\newcommand*{\KL}{\mathrm{KL}}

\ShortHeadings{}{}

\firstpageno{1}

\begin{document}

\title{Stability and Optimization Error of Stochastic Gradient Descent for Pairwise Learning}
\author{Wei Shen\email 16482530@life.hkbu.edu.hk \\
       \addr  Department of Mathematics, Hong Kong Baptist University,\\
Kowloon Tong, Kowloon, Hong Kong
       \AND
       \name Zhenhuan Yang \email zyang6@albany.edu \\
       \addr Department of Mathematics and Statistics\\
        State University of New York at Albany\\
       Albany, USA
       \AND 
       \name Yiming Ying\email yying@albany.edu \\
       \addr Department of Mathematics and Statistics\\
        State University of New York at Albany\\
       Albany, USA
       \AND 
       \name Xiaoming Yuan \email xmyuan@hku.hk\\
       \addr  Department of Mathematics,
The University of Hong Kong,\\
Hong Kong
       }

\editor{}

\maketitle
%
 
%

\editor{}

\maketitle

\begin{abstract}
In this paper we study the stability and its trade-off with optimization error for stochastic gradient descent (SGD) algorithms in the pairwise learning setting.  Pairwise learning  refers to a learning task which involves a loss function depending on pairs of instances among which notable examples are bipartite ranking, metric learning, area under ROC (AUC) maximization and minimum error entropy (MEE) principle.  Our contribution is twofold. Firstly, we establish the stability results of SGD for pairwise learning in the convex, strongly convex and non-convex settings, from which generalization bounds can be naturally derived. Secondly, we establish the trade-off between stability and optimization error of SGD algorithms for pairwise learning. This is achieved by
lower-bounding the sum of stability and optimization error by the minimax
statistical error over a prescribed class of pairwise loss functions.  From this fundamental trade-off, we obtain lower bounds for the optimization error of SGD algorithms and the excess expected risk over a class of pairwise losses. 
In addition, we illustrate our stability results by giving some specific examples of AUC maximization, metric learning and MEE.   
\end{abstract}


\section{Introduction}
 
This paper concerns with {\em pairwise learning} which usually involves a {pairwise loss} function, i.e. the loss function
depends on a pair of examples which can be expressed by
$\ell(h,(x,y), (x',y'))$ for a hypothesis function $h: \X
\to \R.$  This is in contrast to the problem of {\em pointwise learning} in standard classification and regression which typically involves a
univariate loss function $\ell(h,x,y).$  Several important learning tasks can be viewed as pairwise
learning problems. For instance, bipartite
ranking \citep{Agarwal,Clem,Rejchel} and AUC maximization \citep{GJZ,Joachims,YWL,Xinhua,Zhao} aim to correctly predict the ordering of pairs of
binary labeled samples. This involves the use of a misranking
loss  $\ell(h,(x,y), (x',y')) = \mathbb{I}_{\{h(x)-h(x')<0\}} \mathbb{I}_{y=1}\mathbb{I}_{y'=-1},$
where $\mathbb{I}(\cdot)$ is the indicator function.
In practice, one usually replaces the indicator function $\mathbb{I}_{\{h(x)-h(x')<0\}}$ by a smooth convex surrogate function like $(1-(h(x)-h(x')))^2$.
Other important examples include metric learning \citep{bellet2015robustness,Davis,Weinberger,Xing,YL,blake1998uci} and minimum error entropy (MEE) principle \citep{Hu2013,hu2016convergence,principe,wangcheng}.

Stochastic gradient descent (SGD)  has now become the workhorse in machine learning as it scales well to big data.   In particular, SGD-type algorithms for pairwise learning have been proposed and extensively studied in the recent work \citep{GJZ,hu2016convergence,Kar,lin2017online,Wang,Ying,Zhao}. The overall performance of SGD algorithms is measured by  the excess expected risk which can be decomposed into two parts: the {\em optimization error} and {\em generalization error}. The optimization error is sometimes referred to as computational error which characterizes the discrepancy between an output of SGD and  the empirical risk minimizer from batch learning. It portrays how fast the algorithm convergence as the number of iterations grows.  The generalization error describes the discrepancy between the population risk of an output of SGD and its empirical risk. One can interpret the expected   and empirical risks as the test error and the training error, respectively.   The analysis of optimization and generalization errors has been conducted in the existing literature using various approaches but most of them have been done separately.  A natural question would be what is the trade-off between generalization and optimization errors which requires to analyze these two errors together rather than separately.  

Generalization analysis has been done for SGD algorithms for pairwise learning  using different techniques such as covering number \citep{Wang}, Rademacher complexities \citep{kar2012random} and integral operators \citep{hu2016convergence,lin2017online,Ying}.  An alternative approach  is to use the the concept of algorithmic stability \citep{bousquet2002stability,mukherjee2006learning}. While a large amount of work has been devoted to studying the stability for pointwise learning,  there is few work on the stability for pairwise learning except the work by \cite{agarwal2005stability}  which focused on  the regularized ERM formulation for bipartite ranking.

\vspace{2mm}
\noindent {\bf Main Contribution.} The first contribution of our work is to establish random-uniform stability \citep{elisseeff2005stability} of randomized SGD algorithms for pairwise learning in both convex and non-convex settings, from which generalization error bounds of SGD algorithms can be obtained very naturally. We then illustrate the stability results using concrete examples in metric learning, AUC maximization and MEE principle.  Our second contribution is the trade-off framework for stability and optimization error  of SGD for pairwise learning,  which indicates that tight stability leads to a slow convergence rate (large optimization error), and vice versa.  This is achieved by establishing minimax statistical error for  the sum of stability and optimization error over a prescribed class of pairwise loss functions. 
 To the best of our knowledge, this is the first-ever known work on the stability  and its trade-off with optimization error for randomized SGD algorithms in the setting of pairwise learning.   

Our work is inspired by the recent work  \citep{hardt2015} and \citep{chen2018stability} which focused on the setting of pointwise learning.  Our studies differ from previous work in the following aspects.  Firstly, \cite{hardt2015} established stability results for the last iterate of randomized iterative SGD algorithms for pointwise learning. Our work significantly extends the results in \citep{hardt2015} to the setting of pairwise learning since we establish both the last iterate and the average of iterates of SDG algorithms for pairwise learning.  Secondly, \cite{chen2018stability} studied the trade-off results between stability and optimization error for SGD in pointwise learning which employed  a strong notion of stability called {\em uniform stability} \citep{bousquet2002stability} specifically tailored for deterministic algorithms.  Our trade-off framework uses a weak notion called {\em random uniform stability} \citep{elisseeff2005stability} which applies to the randomized iterative SGD algorithms. In addition, we established lower bounds of the average of the iterates  of SGD algorithms for pairwise learning which match the upper bounds in the literature of online pairwise learning \citep{Kar,Wang}. The results are new even for the case of pointwise learning.

\vspace{2mm}
\noindent {\bf Related Work.}   The stability analysis dates back to the work  \citep{devroye1979distribution,rogers1978finite} where it was shown that the variance of the leave-one-out error can be upper bounded by hypothesis stability \citep{kearns1999algorithmic}.  \cite{bousquet2002stability} used the notation of uniform stability and studied stability of regularization based algorithms.  \cite{kutin2012almost} introduced several weaker variants of stability, and showed how they are
sufficient to obtain generalization bounds for certain algorithms.  \cite{rakhlin2012making} and \cite{mukherjee2006learning} studied the relation between stability  and learnability. All these work considered the stability of deterministic learning algorithms such as kNN rules, ERM and regularized network and it cannot be used to study a large number of randomized learning algorithms.   More recently, \cite{chen2018stability} employed the strong notation of uniform stability and established the trade-off between stability and convergence rates  of certain iterative algorithms.

 \cite{elisseeff2005stability} extended the work \citep{bousquet2002stability} and  introduced a notion of random uniform stability for studying randomized algorithms such as bagging.   
\cite{hardt2015} first established random uniform stability  for randomized iterative SGD algorithms for convex and non-convex settings in  pointwise learning. The results were further improved in the work  \citep{kuzborskij2017data,pensia2018generalization} by exploring the structures of the loss function and the data.  

Concurrently, SGD algorithms for pairwise learning were originally introduced and studied in \citep{Wang}.  Pairwise learning involves statistically dependent pairs of instances while, in practice, the individual instances are i.i.d. according an unknown distribution.  As such, standard analysis for the pointwise learning case can not be directly applied to pairwise learning. Indeed, there is a considerable efforts on developing various new techniques to study the convergence of SGD for pairwise learning. In particular,    generalization bounds \citep{cesa2004generalization} of SGD for pairwise  were established using uniform convergence approaches such as  covering number \citep{Wang} and Rademacher complexity \citep{Kar}.  The work \citep{YZ} used  integral operators developed in \citep{rosasco2010learning,smale2007learning} to show the convergence of SGD for pairwise learning with focus on the least-square loss and the setting of reproducing kernel Hilbert spaces.     

A close related concept to algorithmic stability is the statistical robustness which considers the problem of how the estimators change relatively to the perturbation of the underlying distribution generating the data. This robustness concept is more general than algorithmic stability we consider here.  In the appealing work by \cite{christmann2016robustness}, it was shown that minimizers of the regularized ERM is statistically robust in the setting of reproducing kernel Hilbert spaces.

 \vspace{2mm}
 \noindent {\bf Organization of this paper.} The rest of the paper is organized as follows. Section \ref{sec:preliminary} introduces some basic notations and concepts related to stability which will be used later.   In Section \ref{sec:stability-analysis}, we present stability results for SGD in the pairwise learning setting.  We establish the trade-off results between stability and optimization error in Section \ref{sec:trade-off}.  Examples are given in Section \ref{sec:example}.  We conclude the paper in Section \ref{sec:conclusion}.

\section{Preliminaries}\label{sec:preliminary}
Let the sample $S=\{z_i=(x_i,y_i):  i=1,\ldots, n\}$ be drawn i.i.d. from $\mathcal{D}$ on $\Z = \X\times \Y$ where $\X$ is a domain in $\R^d$ and $\Y \subseteq \R.$ 
Let $\bw\in\R^d$ be  the model parameter associated with the hypothesis function $h$ (e.g., the linear hypothesis function $h(x) = \bw^T  x$ ). 
The goal of pairwise learning is to minimize the following population risk: 
\begin{equation}
R(\bw)\overset{\text{def}}{=}\mathbb{E}_{(z,z') \sim \mathcal{D}\times\mathcal{D}}   [\ell(\bw, z,z')].
\end{equation}
The corresponding empirical risk is defined by 
\begin{equation}
R_{S}(\bw)\overset{\text{def}}{=}\frac{2}{n(n-1)}\sum_{i<j}\ell(\bw,z_i,z_j).
\end{equation}
We use the conventional notation $A$ denote the randomized SGD algorithm and $A(S)$ to denote its output based on $S$.  The expected generalization error of $A(S)$ is given by 
\begin{equation}
\epsilon_{\text{gen}}\overset{\text{def}}{=}\mathbb{E}_{S,A}[R_S(A(S))-R(A(S))],
\end{equation}
where the expectation is taken over the randomness of $A$ and $S$.

\subsection{SGD for Pairwise Learning}
Recall that the pairwise learning loss $\ell: \R^d \times \Z \times \Z \to \R^+$ is defined, for any $\bw, z,z'\in\Z$, by $\ell(\bw,z,z').$  The SGD updates for pairwise learning \citep{Kar,Wang,Ying,Zhao} are given by $\bw_1=0$, and for $2\leq t \leq T$,
\begin{equation}\label{eq:sgd}
\bw_{t}=\bw_{t-1}-\frac{\alpha_{t -1}}{t-1}\ \sum_{j=1}^{t-1}\nabla \ell(\bw_{t-1},z_{\xi_t},z_{\xi_j}), 
\end{equation}
where $\{z_{\xi_j}\}_{j=1}^T$ are examples from $S$ with the indexes $\{\xi_j\}_{j=1}^T$ chosen at random from $\{1,\cdots,n\}$, and $\nabla\ell$ denotes the gradient with respect to the first argument.

The above algorithm is an extension of the standard SGD in the pointwise learning setting to the pairwise learning setting.  It was first introduced by  \cite{Wang} as  online gradient descent for pairwise learning. It was further developed   for AUC maximization \citep{GJZ,Ying,Zhao} and MEE \citep{hu2016convergence} for the stochastic setting (i.e. the data are assumed to be i.i.d.). For simplicity, we refer to it as SGD for pairwise learning or just SGD when it is clear from the context.

There are two schemes for choosing $\{\xi_j\}_{j=1}^T$ for the SGD update rule which are independent of the sample $S$.
The first one, called the {\em random permutation rule}, is to choose a new random permutation over $\{1,\cdots,n\}$ at the beginning of each epoch and go through the examples in the order determined by the permutations.
The other is the {\em random selection rule} which  selects each $\xi_j$ uniformly at random in $\{1,\cdots,n\}$ at each step.  In this work, our results hold true for the above two schemes. 

The output of SGD algorithm \eqref{eq:sgd} at $T$ can be the last iterate $A(S) = \bw_T$ or the average of iterates $A(S) = \bar{\bw}_T =\frac{1}{T}  \sum_{t=1}^T \bw_t.$  We denote $\ASL = \bw_T$ and $\ASA = \bar{\bw}_T.$ Later on we use the conventional notation $A(S)$ to denote it can be  either $\ASA$ or $\ASL$.

\subsection{Algorithmic Stability and Its Relation with Generalization}\label{sec:error-decomp}
We will use a modification of $\epsilon$-uniform stability introduced by \cite{Agarwal} which considered the regularized ERM formulation for ranking problems. It can also be regarded as an extension of random uniform stability \citep{elisseeff2005stability} to the case of pairwise learning. 

\begin{definition}\label{def:stability}
An SGD algorithm $A$ for pairwise learning is called random uniform stable with $\gep>0$ if for all data sets $S, S'\in \Z^n$ according to distribution $D$ such that $S$ and $S'$ differ in at most one example, we have
\begin{equation}\label{eq:def-stab}
\sup_{(z,z')\sim \mathcal{D}\times\mathcal{D}} \mathbb{E}_{A} [\ell(A(S),z,z')-\ell(A(S'),z,z')]\leq \epsilon.
\end{equation}
Here, the expectation is taken only over the randomness of $A$. We denote the smallest constant $\gep$ satisfies \eqref{eq:def-stab} as  $\epsilon_{\text{stab}}(A,T,\ell,D,n).$ 
\end{definition}It is worthy of noting that we always assume that the randomness for algorithm $A$ is independent of the sample $S$ which is i.i.d. generated from $D$ on $\X\times \Y.$  The notation $\epsilon_{\text{stab}}(A,T,\ell,D,n)$ can be $\epsilon_{\text{stab}}(A^{\text{last}},T,\ell,D,n)$ for the last iterate of SGD or $\epsilon_{\text{stab}}(A^{\text{avg}},T,\ell,D,n)$ for the average of iterates. 
 
The following theorem describes the relation between the stability and generalization for pairwise learning which is originally in the work \cite{Agarwal,agarwal2005stability} for bipartite ranking.  We include its proof for completeness. 

\begin{theorem}\label{thm:stable}
If the SGD algorithm $A$ is random uniform stable with $\gep>0$,  then we have 
\begin{equation}
|\mathbb{E}_{S,A} [R_{S}(A(S))-R(A(S))]|\leq 2\epsilon.
\end{equation}
 \end{theorem}
\begin{proof} 
Denote by $S = (z_1,\cdots,z_n)$ and $\tilde{S} = (\tilde{z}_1,\cdots,\tilde{z}_n)$ two samples wherein the examples are i.i.d. chosen from $\D$.
Let $S'(i)$   be an i.i.d. copy of $S$ except the $i$th example being replaced by $\tilde{z}_i$.
Let  $S''(i,j) = (z_1, \cdots, \tilde{z}_i, \cdots , \tilde{z}_j,\cdots,z_n)$. Therefore, \begin{align*}
\label{eq:thm:stable:1} 
&\mathbb{E}_S \mathbb{E}_A [R_S(A(S))]=\mathbb{E}_S \mathbb{E}_A \Big[\frac{2}{n(n-1)}\sum_{i<j}\ell(A(S);z_i,z_j)\Big] \\
&{=}\mathbb{E}_{\tilde{S}}  \mathbb{E}_S \mathbb{E}_A \Big[\frac{2}{n(n-1)}\sum_{i<j}\ell(A(S''(i,j));\tilde{z}_i,\tilde{z}_j)\Big] \\
&=\mathbb{E}_{\tilde{S}}  \mathbb{E}_S \mathbb{E}_A \Big[\frac{2}{n(n-1)}\sum_{i<j}\ell(A(S);\tilde{z}_i,\tilde{z}_j)\Big]+\delta =\mathbb{E}_S \mathbb{E}_A \Big[R(A(S))\Big]+\delta,\numberthis
\end{align*}
where the second equality comes from the identical distribution assumption.  The residual term $\delta$ in the last two equations  can be expressed as
\begin{align*}\label{eq:thm:stable:2}
&\delta=\frac{2}{n(n-1)}\sum_{i<j}\mathbb{E}_{\tilde{S}}  \mathbb{E}_S \mathbb{E}_A \Big[\ell(A(S''(i,j));\tilde{z}_i,\tilde{z}_j)-\ell(A(S);\tilde{z}_i,\tilde{z}_j)\Big]\\
&=\frac{2}{n(n-1)}\sum_{i<j}\mathbb{E}_{\tilde{S}}  \mathbb{E}_S \mathbb{E}_A \Big[\ell(A(S''(i,j));\tilde{z}_i,\tilde{z}_j)-\ell(A(S'(i));\tilde{z}_i,\tilde{z}_j)\\
&+\ell(A(S'(i));\tilde{z}_i,\tilde{z}_j)-\ell(A(S);\tilde{z}_i,\tilde{z}_j)\Big]\\
&=\frac{2}{n(n-1)}\sum_{i<j}\mathbb{E}_S \mathbb{E}_A \mathbb{E}_{(\tilde{z}_i,\tilde{z}_j)\sim \mathcal{D}\times\mathcal{D}} \Big[\ell(A(S''(i,j));\tilde{z}_i,\tilde{z}_j)-\ell(A(S'(i));\tilde{z}_i,\tilde{z}_j)\Big]\\
&+\frac{2}{n(n-1)}\sum_{i<j}\mathbb{E}_S \mathbb{E}_A \mathbb{E}_{(\tilde{z}_i,\tilde{z}_j)\sim \mathcal{D}\times\mathcal{D}}\Big[\ell(A(S'(i));\tilde{z}_i,\tilde{z}_j)-\ell(A(S);\tilde{z}_i,\tilde{z}_j)\Big].\numberthis
\end{align*}
Note that $S''(i,j)$ and $S'(i)$ differ in only one example and so do $S'(i)$ and $S$.
Furthermore, taking the supremum over any two data sets $S,S'$ differing in only one example, we can bound the difference as
\begin{eqnarray}
|\delta|\leq 2\sup_{S,S',(z,\tilde{z})\sim D\times D}  \mathbb{E}_A  \left[\ell(A(S');z,\tilde{z})-\ell(A(S);z,\tilde{z})\right]\leq 2\epsilon,
\end{eqnarray}
by our assumption on the random uniform stability of $A$. The claim follows.
\end{proof}
Theorem \ref{thm:stable} bounds the expected generalization error of SGD for pairwise learning with two times of its random uniform stability bound.  We will present the detailed bounds for the stability of SGD for pairwise learning in Section \ref{sec:stability-analysis}.

\subsection{Stability and Optimization Error Decomposition}
In this subsection, we assume $\bw \in \Omega \subseteq \mathbb{R}^d$. 
Recall that $A(S)$ is the output of SGD algorithm \eqref{eq:sgd} for pairwise learning  at iteration $T$. 
The overall performance of the output $A(S)$  is measured in terms of the {\em excess risk} defined as \begin{equation}\label{eq:excess-risk}
\Delta R(A(S)) \overset{\text{def}}{=}R(A(S))-\inf_{\bw\in \Omega}R(\bw).
\end{equation}
For notional simplicity, let 
 \begin{equation}\label{eq2-4}
 {\bw}^{*}_S=\argmin_{\bw\in\Omega}R_S(\bw),
 \end{equation} 
 and 
 \begin{equation}
 \bw^{*}=\argmin_{\bw\in\Omega}R(\bw).
 \end{equation}
Then we can obtain the following decomposition, namely,
\begin{eqnarray}\label{eq2-6}
\Delta R(A(S)) &=& R(A(S))-R(\bw^{*})\nonumber\\
&=& R(A(S))-R_S(A(S))+R_S(A(S))-R_S({\bw}^{*}_S)\nonumber\\
&&+R_S({\bw}^{*}_S)-R_S(\bw^{*})+R_S(\bw^{*})-R(\bw^{*})\nonumber\\
&\leq& R(A(S))-R_S(A(S))+R_S(A(S))-R_S({\bw}^{*}_S)\nonumber\\
&&+R_S(\bw^{*})-R(\bw^{*}),
\end{eqnarray}
where the last inequality follows from the fact $R_S({\bw}^{*}_S)-R_S(\bw^{*})\le 0$ from the definition  ${\bw}^{*}_S$ (i.e. \eqref{eq2-4}).
Taking expectation on both sides of $(\ref{eq2-6})$ w.r.t. the randomness of $S$ and $A$ and noting that $\mathbb{E}_S[R_S(\bw^{*})-R(\bw^{*})]=0$, we can decompose the expected excess risk as
\begin{eqnarray}\label{eq2-7}
\hspace*{-5mm}\mathbb{E}_{S,A}[\Delta R(A(S))] 
\hspace*{-1mm}\leq \hspace*{-1mm}\mathbb{E}_{S,A}[\underbrace{R(A(S))-R_S(A(S))}_{\text{generalization error}}]\hspace*{-1mm}+\hspace*{-1mm}\mathbb{E}_{S,A}[\underbrace{R_S(A(S))-R_S({\bw}^{*}_S)}_{\text{optimization error}}].
\end{eqnarray}
Denote the expected generalization error and optimization error of $A(S)$ as 
$
\epsilon_{\text{gen}}(A,T,\ell,D,n)\overset{\text{def}}{=}\mathbb{E}_{S,A}[R(A(S))-R_S(A(S))]
$
and 
$
\epsilon_{\text{opt}}(A,T,\ell,D,n)\overset{\text{def}}{=}\mathbb{E}_{S,A}[R_S(A(S))-R_S({\bw}^{*}_S)].
$
Note that the above quantities are indexed by the estimator $A(S)$, loss function $\ell$, data distribution $D$ and sample size $n$. When it is clear from the context, we will omit these indexes for simplicity. 
As a result, we can rewrite $(\ref{eq2-7})$ as
\begin{equation}\label{eq2-decom}
\mathbb{E}_{S,A}[\Delta R(A(S))] \leq \epsilon_{\text{gen}}(A,T,\ell,D,n)+\epsilon_{\text{opt}}(A,T,\ell,D,n).
\end{equation}
Combining the expected excess risk decomposition $(\ref{eq2-decom})$ and Theorem \ref{thm:stable}, we have, for any loss $\ell$, that 
\begin{equation}\label{eq3-1}
\mathbb{E}_{S,A}[\Delta R(A(S)] \leq 2\epsilon_{\text{stab}}(A,T,\ell,D,n)+\epsilon_{\text{opt}}(A,T,\ell,D,n).
\end{equation}
The above inequality means that the overall performance of SGD measured by the excess population risk $\Delta R(A(S))$ can be decomposed into stability and optimization error.  This leads to a natural question that what is the trade-off between these two terms and whether SGD can achieve both the tighter stability bounds and fast convergence rate.

To answer this questions,  we consider the stability and optimization error for the last output of SGD (i.e.   $\ASL$) over a class of convex pairwise losses  $\L$ and $\D$ is the class of all probability distributions which are given by 
$$\SLST(T,\mathcal{L},\mathcal{D},n)\overset{\text{def}}{=}\sup_{\ell\in\mathcal{L},D\in\mathcal{D}} \epsilon_{\text{stab}}(A^{\text{last}},T,\ell,D,n),$$
and 
$$\OLST(T,\mathcal{L},\mathcal{D},n)\overset{\text{def}}{=}\sup_{\ell\in\mathcal{L},D\in\mathcal{D}} \epsilon_{\text{opt}}(A^{\text{last}},T,\ell,D,n). $$
Likewise, one can define $$\SAVG(T,\mathcal{L},\mathcal{D},n)\overset{\text{def}}{=}\sup_{\ell\in\mathcal{L},D\in\mathcal{D}} \epsilon_{\text{stab}}(\ASA,T,\ell,D,n)$$ and $$\OAVG(T,\mathcal{L},\mathcal{D},n)\overset{\text{def}}{=}\sup_{\ell\in\mathcal{L},D\in\mathcal{D}} \epsilon_{\text{opt}}(\ASA,T,\ell,D,n). $$

Recall that the minimax risk in nonparametric statistics \citep{Tsybakov,wainwright2019high} is given by $\inf_{\hbw}\sup_{D\in\mathcal{D}}\mathbb{E}_{S\sim  D^n} [\Delta R(\hbw)]$ where the infinimum is taken with respect to all possible estimator $\hbw:  \Z^n \to \R^d$ which is a function of a random sample $S=\{z_1, \ldots, z_n\}$, i.e. $\hbw = \hbw(S).$ The key idea is to connect the above two errors with minimax risk in nonparametric statistics as given by the following lemma. 
\begin{lemma}\label{lem:trade-off} 
For any convex pairwise loss $\ell \in \L$, there holds 
\begin{equation}\label{eq3-2-tradeoff}2\SLST(T,\mathcal{L},\mathcal{D},n)+\OLST(T,\mathcal{L},\mathcal{D},n)\geq\inf_{\hbw}\sup_{D\in\mathcal{D}}\mathbb{E}_{S\sim  D^n} [\Delta R(\hbw)], 
\end{equation}
and
\begin{equation}\label{eq3-2-tradeoff2}2\SAVG(T,\mathcal{L},\mathcal{D},n)+\OAVG(T,\mathcal{L},\mathcal{D},n)\geq\inf_{\hbw}\sup_{D\in\mathcal{D}}\mathbb{E}_{S\sim  D^n} [\Delta R(\hbw)].
\end{equation}

\end{lemma}
\begin{proof}  We only prove \eqref{eq3-2-tradeoff} as the proof for \eqref{eq3-2-tradeoff2} is exactly the same. 

  From \eqref{eq3-1} and definitions for $\SLST(T,\mathcal{L},\mathcal{D},n)$ and $\OLST(T,\mathcal{L},\mathcal{D},n)$, we have, for any $\ell\in \L$, that 
\begin{equation}\label{eq:inter-1}\sup_{D\in \D}  \EX_{S, A}[\gD R( \ASL) ]\le 2\SLST(T,\mathcal{L},\mathcal{D},n)+\OLST(T,\mathcal{L},\mathcal{D},n).\end{equation}
Notice that $\gD R (\ASL )=  \EX_{(z,z')} [\ell(\ASL, z,z')]  - \inf_{\bw}\EX_{(z,z')}[ \ell(\bw, z,z')]$ and the randomness of the SGD algorithm $A$ is independent of $S$.  Consequently, 
\begin{align*}\label{eq:inter-2}& \EX_{S, A} [\gD R(\ASL ) ]  = \EX_{S} \bigl\{ \EX_{A}[\gD R( \ASL ) ] \bigr\} \\ &  = \EX_{S}\bigl \{ \EX_{A}[ \EX_{(z,z')} [\ell(\ASL, z,z')] \bigr\}  - \inf_{\bw}\EX_{(z,z')}[ \ell(\bw, z,z')]. \numberthis \end{align*} Since $\ell\in \L$ is convex with respect to the first argument,   Jensen's inequality tells us that 
\begin{equation}\label{eq:inter-3}\EX_{A}[ \EX_{(z,z')} [\ell(\ASL, z,z')] \ge  \EX_{(z,z')} \bigl[\ell (\EX_{A}[ \ASL ], z,z')\bigr].\end{equation}  Putting \eqref{eq:inter-2} and \eqref{eq:inter-3} together, we have 
\begin{align*} 
	\EX_{S, A} [\gD R( \ASL ) ] \ge \EX_{S} \bigl[\gD R( \, \EX_A [ \ASL ] \,) \bigr].
\end{align*}
Putting this back into \eqref{eq:inter-1} yields that 
\begin{align*}
2\SLST(T,\mathcal{L},\mathcal{D},n)+\OLST(T,\mathcal{L},\mathcal{D},n) & \geq   \sup_{D\in \D} \EX_{S} \bigl[\gD R( \, \EX_A [ \ASL ] \,) \bigr] \\
& \ge \inf_{\hbw}\sup_{D\in\mathcal{D}}\mathbb{E}_{S\sim  D^n} [\Delta R(\hbw)].	
\end{align*}
This completes the proof of the lemma. 
\end{proof}
Using techniques from nonparametric statistics (e.g. \cite{le2012asymptotic,Tsybakov,wainwright2019high}), one can estimate the minimum risk on the righthand side of \eqref{eq3-2-tradeoff} and thus derive trade-off results between stability and optimization error of SGD for pairwise learning as we will do soon  in Section \ref{sec:trade-off}.

It is worth of mentioning that this connection \eqref{eq3-2-tradeoff} was first observed by \cite{chen2018stability} for pointwise learning which, however, focused on the deterministic algorithms. Specifically, the uniform stability in \citep{chen2018stability} is not taken with respect to the randomness of algorithm $A$ and the expectation $\EX$ involved in  Lemma \ref{lem:trade-off} is only with respect to $S$ without the randomness of algorithm $A$.  Our paper studies stability of SGD algorithm defined by \eqref{eq:sgd} which involves the randomness of $\{\xi_j\}$, and the uniform stability defined by Definition \ref{def:stability} is taken in the sense of the expectation of $\{\xi_j\}.$ In this sense, our result stated in Lemma \ref{lem:trade-off} is a non-trivial extension of \citep{chen2018stability} to the the case of randomized SGD algorithms for pairwise learning.

\section{Stability Analysis of SGD Algorithms}\label{sec:stability-analysis}
In this section we establish stability results for SGD algorithms given by \eqref{eq:sgd}. Before we present the main stability results, we introduce some definitions and background materials. 

\subsection{Warm-up: Some Technical Preparation}
  The following definitions list convexity and smoothness properties of a function $f$.
\begin{definition}\label{def:convex}
A function $f$ is convex if and only if $\dom f$ is a convex set and 
$
f(\theta x_1 +(1-\theta) x_2)\leq \theta f(x_1)+(1-\theta)f(x_2),
$
for all $x_1, x_2 \in \dom f$ and $\theta \in[0,1]$.
And 
a function $f$ is $\gamma-$strongly convex if and only if 
$g(x)=f(x)-(\gamma/2)x^{\top}x$
is convex.
\end{definition}
\begin{definition}\label{def:smooth}
A function $f$ is $L-$Lipschitz if and only if
$\|f(x_2)-f(x_1)\|\leq L \cdot\|x_2-x_1\|$, for all $x_1, x_2 \in \dom f.$ Furthermore, 
a function $f$ is $\beta-$smooth if and only if $f$ is differentiable and $\nabla f(x)$ is $\beta$-Lipschitz.
\end{definition}

Let $S'=\{z'_1,z'_2,\cdots,z'_n\}$ be an i.i.d.  copy of $S$ but  differ from $S$ at precisely one location. 
Assume  SGD for pairwise learning is run based on $S$ and $S'$ along the same path $\{\xi_1,\xi_2,\cdots,\xi_T\}$ with the same initial points $\bw_1=\bw'_1=0.$ Recall, for $t=2,\cdots,T$,  the SGD updates based on $S$ are given by
\begin{eqnarray}\label{eq:SGDupdate}
G_t(\bw_{t-1})&=&\bw_{t-1}-\frac{\alpha_{t -1}}{t-1}\ \sum_{j=1}^{t-1}\nabla \ell(\bw_{t-1},z_{\xi_t},z_{\xi_j}).
\end{eqnarray}
Similarly, for $t=2,\cdots,T$, we denote the gradient updates based on $S'$ by
\begin{eqnarray*}
G'_t(\bw'_{t-1})&=&\bw'_{t-1}-\frac{\alpha_{t -1}}{t-1}\ \sum_{j=1}^{t-1}\nabla \ell(\bw'_{t-1}, z'_{\xi_t},z'_{\xi_j}).
\end{eqnarray*}
We say that an operator $G_t$ is {\em expansive} with parameter $\eta_t>0$ if $\|G_t(\bw)-G_t(\bw')\| \le \eta_t\|\bw-\bw'\|$ for any $\bw$ and $\bw'.$   The main theorems about stability  rely on the following lemma which states $G_t$ is expansive.

\begin{lemma}\label{lem:expansive}
Assume that $\ell(\cdot,z,z')$ is $\beta-$smooth for every pair $(z,z')$.
 \begin{enumerate}[label=(\alph*)]
\item Then $G_t$ is $(1+\alpha_{t-1}\beta)$-expansive.
\item Assume in addition that $\ell(\cdot, z,z')$ is convex and $\alpha_{t-1} \leq \frac{2}{\beta}$. Then $G_t$ is 1-expansive.
\item Assume in addition that $\ell(\cdot,z,z')$ is $\gamma$-strongly convex and $\alpha_{t-1}\leq \frac{2}{\beta+\gamma}$. Then $G_t$ is $\left(1-\frac{\beta\gamma\alpha_{t-1}}{\beta+\gamma}\right)$-expansive.
\end{enumerate}
\end{lemma}
The proof for the above  elementary results can be found in \ref{app:A}.  Note that the results of Lemma \ref{lem:expansive} about $G_t$ also apply to $G'_t$.


Now consider the SGD updates respectively on $S$ and $S'$ with $\bw_t = G_t(\bw_{t-1})$ and $\bw'_t = G_t'(\bw'_{t-1})$  for any $t\ge 2$ and the initial point $\bw_1 = \bw'_1=0.$ The stability of SGD for pairwise learning critically depends on the following recursive property of  $\gd_t = \|\bw_t - \bw'_t\|.$

\begin{theorem}\label{thm:recursive} Assume that $\ell(\cdot,z,z')$ is $L$-Lipschitz for any $z, z'.$ Suppose that both $G_t$ and $G'_t$ are expansive with parameter  $\eta_t.$  Then for $1< t \leq T$, under both random rules (e.g. random permutation or selection rules), the following recursive relation  holds true. 
\begin{eqnarray}
\mathbb{E}[\delta_t]\leq \Big\{\frac{1}{n}\cdot\min(\eta_t,1)+\Big(1-\frac{1}{n}\Big)\cdot\eta_t\Big\}\mathbb{E}[\delta_{t-1}]+\frac{4L}{n}\cdot \alpha_{t-1}.
\end{eqnarray}
\end{theorem}

The proof of this theorem is inspired by the work \citep{hardt2015}.  However, compared with the situation in the context of pointwise learning, the key challenge here is that at any step $t$, the computation of the new gradient direction not only depends on the current example $z_{\xi_t}$ but also on all previously used examples, i.e. $\{z_{\xi_i}\}_{i=1}^{t-1}$.
We overcome this hurdle by a careful investigation into how many times SGD has encountered the different examples between $S$ and $S'$ before the $t$-th step, as illustrated below respectively for both cases of random selection and permutation rules. 

We first consider the case of  random selection rule. 
\begin{lemma}\label{lem:prop2}
Suppose that we run SGD based on $S$ and $S'$ under the random selection rule for $T$ steps along the same path $\{\xi_1,\xi_2,\cdots,\xi_T\}$.
For a fixed $t\in(1,T]$,
assume among the first $t-1$ steps, there are $m$ steps where  SGD has encountered the different examples. 
Then we have the following properties:
\begin{enumerate}
\item $\delta_t \leq \min(\eta_t,1)\delta_{t-1} + 2\alpha_{t-1}L$, if $z_{\xi_t}\neq z'_{\xi_t}$;
\item $\delta_t \leq \eta_t\delta_{t-1} + \frac{m}{t-1}\cdot 2\alpha_{t-1}L$, if $z_{\xi_t}=z'_{\xi_t}$,
\end{enumerate}
wherein $\eta_t$ is the expansive parameter of the updates $G_t$ and $G'_t$. 
\end{lemma} 
\begin{proof}
First of all, for  either case, we have 
\begin{align*}\label{eq:basic-espansive}
&\delta_t = \|G_t(\bw_{t-1})-G'_t(\bw'_{t-1})\| \\
&\leq \|G_t(\bw_{t-1})-G_t(\bw'_{t-1})\|+\|G_t(\bw'_{t-1})-G'_t(\bw'_{t-1})\| \\
&\leq \eta_t\delta_{t-1} +\frac{\alpha_{t-1}}{t-1}\sum_{j=1}^{t-1}\|\nabla_w \ell(\bw'_{t-1},z'_{\xi_t},z'_{\xi_j})-\nabla_w \ell(\bw'_{t-1},z_{\xi_t},z_{\xi_j})\|. \numberthis
\end{align*}
Then we prove the two claims in this lemma separately.

1)~ For the first property, if $z_{\xi_t}\neq z'_{\xi_t}$, we have $\nabla_w \ell(\bw'_{t-1},z'_{\xi_t},z'_{\xi_j})\neq\nabla_w \ell(\bw'_{t-1},z_{\xi_t},z_{\xi_j})$ for all $j=1,\cdots,t-1$. Then following the $L-$Lipschitz condition of $\ell$, we have
$$\|\nabla_w \ell(\bw'_{t-1},z'_{\xi_t},z'_{\xi_j})-\nabla_w \ell(\bw'_{t-1},z_{\xi_t},z_{\xi_j})\|\leq 2L.$$
As a result, we obtain
\begin{eqnarray}\label{eq:lemma-rs-1}
\delta_t \leq \eta_t \delta_{t-1}+2\alpha_{t-1}L.
\end{eqnarray}
Next we prove the other half of the first claim of this lemma. By the triangle inequality, we have
\begin{align*}\label{eq:lemma-rs-2}
&\delta_t = \|G_t(\bw_{t-1})-G'_t(\bw'_{t-1})\| \\
&\leq \|\bw_{t-1}-\bw'_{t-1}\|+\frac{\alpha_{t-1}}{t-1}\sum_{j=1}^{t-1}\|\nabla_w \ell(\bw'_{t-1};z'_{\xi_t},z'_{\xi_j})-\nabla_w \ell(\bw_{t-1};z_{\xi_t},z_{\xi_j})\|\\
&\leq\delta_{t-1}+2\alpha_{t-1}L.\numberthis
\end{align*}
Thus the first property follows by combining (\ref{eq:lemma-rs-1}) and (\ref{eq:lemma-rs-2}).

2)~ We now prove the second property. Denote $U=\{1\leq j\leq t-1 | z_{\xi_j}\neq z'_{\xi_j}\}$. From the assumption that there are $m$ steps where SGD has encountered the different examples among the first $t-1$ steps, we know there are $m$ number of elements in  $\{z_{\xi_j}\}_{j=1}^{t-1}$ which are different from those in $\{z'_{\xi_j}\}_{j=1}^{t-1}$. That means $|U|=m$ where $|U|$ is the number of coordinates in the set $U$.
Recall we have $z_{\xi_t}=z'_{\xi_t}$ and thus there are at most $m$ number of the pairs $\{z_{\xi_t},z_{\xi_j}\}_{j=1}^{t-1}$ which are different from $\{z'_{\xi_t},z'_{\xi_j}\}_{j=1}^{t-1}$. It follows that 
\begin{align*}&\sum_{j=1}^{t-1}\|\nabla_w \ell(\bw'_{t-1},z'_{\xi_t},z'_{\xi_j})-\nabla_w \ell(\bw'_{t-1},z_{\xi_t},z_{\xi_j})\|\\
&=\sum_{j\in U}\|\nabla_w \ell(\bw'_{t-1},z'_{\xi_t},z'_{\xi_j})-\nabla_w \ell(\bw'_{t-1},z_{\xi_t},z_{\xi_j})\|.\end{align*}
Thus following the $L-$Lipschitz condition of $\ell$, we have
\begin{align*} \sum_{j=1}^{t-1}\|\nabla_w \ell(\bw'_{t-1},z'_{\xi_t},z'_{\xi_j})-\nabla_w \ell(\bw'_{t-1},z_{\xi_t},z_{\xi_j})\|  \leq 2mL.\end{align*} Plugging this into  \eqref{eq:basic-espansive}, we get the second property.
\end{proof}

Now we consider the permutation rule for $T$ steps. In this case, let $t_k^{\star}=\{t\ |\ z_{\xi_t}\neq z'_{\xi_t},\ (k-1)n< t \leq kn\}$ for each $k\ge 1$.
In fact at the $t_k^{\star}$-th step, SGD encounters the different examples during the $k$-th epoch.
Fix an arbitrary sequence of SGD updates $G_1,\cdots,G_T$ based on $S$ and another sequence $G'_1,\cdots,G'_T$ based on $S'$.  We have the following lemma for the recursive property of the SGD updates. 
\begin{lemma}\label{lem:prop1}
Suppose that we run SGD based on $S$ and $S'$ under the random permutation rule for $T$ steps along the same path $\{\xi_1,\xi_2,\cdots,\xi_T\}$.
Assume that both $G_t$ and $G'_t$ are expansive with parameter  $\eta_t.$ For $(k-1)n<t\leq kn$ where $k$ is the number of epochs, we have the following properties:
\begin{enumerate}
\item $\delta_t \leq \min(\eta_t,1)\delta_{t-1} + 2\alpha_{t-1}L$, if $t=t_k^{\star}$,
\item $\delta_t \leq \eta_t\delta_{t-1} + \frac{k-1}{t-1}\cdot 2\alpha_{t-1}L$, if $(k-1)n<t<t_k^{\star}$,
\item $\delta_t \leq \eta_t\delta_{t-1} + \frac{k}{t-1}\cdot 2\alpha_{t-1}L$, if $t_k^{\star}<t\leq kn$.
\end{enumerate}
\end{lemma} 
\begin{proof}
1)~ For each $(k-1)n<t\leq kn$ where $k$ is the number of epochs, we have  
\begin{align*}\label{eq2-10}
&\delta_t = \|G_t(\bw_{t-1})-G'_t(\bw'_{t-1})\| \\
&\leq \|G_t(\bw_{t-1})-G_t(\bw'_{t-1})\|+\|G_t(\bw'_{t-1})-G'_t(\bw'_{t-1})\| \\
&\leq \eta_t\delta_{t-1} +\frac{\alpha_{t-1}}{t-1}\sum_{j=1}^{t-1}\|\nabla_w \ell(\bw'_{t-1},z'_{\xi_t},z'_{\xi_j})-\nabla_w \ell(\bw'_{t-1},z_{\xi_t},z_{\xi_j})\|. \numberthis
\end{align*}
For the first property, if $t=t_k^{\star}$, we must have $z_{\xi_t}\neq z'_{\xi_t}$. 
As a result, $\nabla_w \ell(\bw'_{t-1},z'_{\xi_t},z'_{\xi_j})\neq\nabla_w \ell(\bw'_{t-1},z_{\xi_t},z_{\xi_j})$ for all $j=1,\cdots,t-1$. Then following the $L-$Lipschitz condition of $\ell$, we have
\begin{eqnarray}\label{eq2-11}
\delta_t \leq \eta_t \delta_{t-1}+2\alpha_{t-1}L.
\end{eqnarray}
Next we prove the other half. By the triangle inequality, we have
\begin{align*}\label{eq2-12}
&\delta_t = \|G_t(\bw_{t-1})-G'_t(\bw'_{t-1})\| \\
&\leq \|\bw_{t-1}-\bw'_{t-1}\|+\frac{\alpha_{t-1}}{t-1}\sum_{j=1}^{t-1}\|\nabla_w \ell(\bw'_{t-1};z'_{\xi_t},z'_{\xi_j})-\nabla_w \ell(\bw_{t-1};z_{\xi_t},z_{\xi_j})\|\\
&\leq\delta_{t-1}+2\alpha_{t-1}L.\numberthis
\end{align*}
Thus the first property follows by combining (\ref{eq2-11}) and (\ref{eq2-12}).

2)~ We now prove the second property. 
If $(k-1)n<t<t_k^{\star}$, we have $z_{\xi_{j}}\neq z'_{\xi_{j}}$  when $ j \in U^{\star}:= \{t_1^{\star},\cdots,t_{k-1}^{\star}\}$, while $z_{\xi_{j}}=z'_{\xi_{j}}$ for  $j$ belonging to $\{1,2,\cdots,t\}$ but not in $U^{\star}$.
As a result, $z_{\xi_t}=z'_{\xi_t}$ and there are at most $(k-1)$ number of the pairs $\{(z_{\xi_t},z_{\xi_j})\}_{j \in U^{\star}}$ which are different from $\{(z'_{\xi_t},z'_{\xi_j})\}_{j \in U^{\star}}$.
Thus following the $L-$Lipschitz condition of $\ell$, we have
\begin{align*}& \sum_{j=1}^{t-1}\|\nabla_w \ell(\bw'_{t-1},z'_{\xi_t},z'_{\xi_j})-\nabla_w \ell(\bw'_{t-1},z_{\xi_t},z_{\xi_j})\| \\ & = \sum_{j\in U^{\star}}\|\nabla_w \ell(\bw'_{t-1},z_{\xi_t},z'_{\xi_{j}})-\nabla_w \ell(\bw'_{t-1},z_{\xi_t},z_{\xi_{j}})\| \leq 2(k-1)L.\end{align*} Plugging this into  \eqref{eq2-10}, we get the second property.

3)~ Following the same strategy as above, if $t_k^{\star}<t\leq kn$, there are at most $k$ number of the pairs $\{(z_{\xi_t},z_{\xi_j})\}_{j \in V^{\star}}$ which are different from $\{(z'_{\xi_t},z'_{\xi_j})\}_{j \in V^{\star}}$, where $V^{\star}=\{t_1^{\star},\cdots,t_{k}^{\star}\}$.
Similarly, we have
$$
\sum_{j=1}^{t-1}\|\nabla \ell(\bw'_{t-1};z'_{\xi_t},z'_{\xi_j})-\nabla 
\ell(\bw'_{t-1};z_{\xi_t}, z_{\xi_j})\|\leq 2k L.
$$
Plugging this into the equation (\ref{eq2-10}), we get the third property.
\end{proof}

We are now in a position to prove Theorem \ref{thm:recursive}. 
 
\noindent{\bf Proof of Theorem \ref{thm:recursive}.}  ~Firstly, under the random selection rule, we denote $m$ as the times of SGD choosing the different examples during the first $t-1$ steps.
Since the examples chosen by SGD at each step are i.i.d. under the random selection rule, $m$ follows a binomial distribution, i.e. $m \sim \mathbb{B}(t-1,1/n)$.
And we know that at step $t$, $\mathbb{P}\{z_{\xi_t}\neq z'_{\xi_t}\}=\frac{1}{n}.$
Then by the independence between the $t-$th step and previous $t-1$ steps,
the probability of that $z_{\xi_t}= z'_{\xi_t}$  at the $t$-th step and SGD has encountered the different examples $m$ times during the previous $t-1$ steps is
$\Big(1-\frac{1}{n}\Big)\cdot C_{t-1}^{m}\Big(1-\frac{1}{n}\Big)^{t-1-m}\Big(\frac{1}{n}\Big)^m$ where $C_{t-1}^{m}$ is the binomial coefficient. 
By  Lemma \ref{lem:prop2}, for every $1<t\leq T$, we have 
\begin{align*}
&\mathbb{E}[\delta_t]
\leq\frac{1}{n}\cdot\Big(\min(\eta_t,1) \mathbb{E}[\delta_{t-1}] + 2\alpha_{t-1}L\Big)\\
&+\sum_{m=0}^{t-1}\left(1-\frac{1}{n}\right)\cdot C_{t-1}^{m}\left(1-\frac{1}{n}\right)^{t-1-m}\left(\frac{1}{n}\right)^m\times \Big(\eta_t \mathbb{E}[\delta_{t-1}] + \frac{m}{t-1}\cdot 2\alpha_{t-1}L\Big)\\
&\leq\Big\{\frac{1}{n}\cdot\min(\eta_t,1)+\Big(1-\frac{1}{n}\Big)\cdot\eta_t\Big\} \mathbb{E}[\delta_{t-1}]+\frac{4L\alpha_{t-1}}{n},
\end{align*}
wherein the second inequality follows from the facts
$$
\sum_{m=0}^{t-1}C_{t-1}^{m}\left(1-\frac{1}{n}\right)^{t-1-m}\left(\frac{1}{n}\right)^m=1
$$
and 
$$
\sum_{m=0}^{t-1}mC_{t-1}^{m}\left(1-\frac{1}{n}\right)^{t-1-m}\left(\frac{1}{n}\right)^m=\frac{t-1}{n}.
$$

Secondly, under the random permutation rule, $t_k^{\star}$ is a uniformly random number in $\{(k-1)n+1,(k-1)n+2,\cdots,kn\}$ and therefore $\forall k \ge 1$,for every $(k-1)n<t\leq kn$ we have
 $$\mathbb{P}\{t_k^{\star}=t\}=\frac{1}{n},\  \mathbb{P}\{t_k^{\star}>t\}=1-\frac{t-(k-1)n}{n}=k-\frac{t}{n},$$ and $$\mathbb{P}\{t_k^{\star}<t\}=\frac{t-1-(k-1)n}{n}=\frac{t-1}{n}-(k-1).$$
By  Lemma \ref{lem:prop1}, for every $(k-1)n<t\leq kn$ with $k\ge 1$, we have
\begin{align*}    
& \mathbb{E}[\delta_t]\leq  \frac{1}{n}\cdot\Big(\min\bigl(\eta_t,1\bigr) \mathbb{E}[\delta_{t-1}] + 2\alpha_{t-1}L\Big)+\left(k-\frac{t}{n}\right)\cdot\left(\eta_t \mathbb{E}[\delta_{t-1}]+ \frac{k-1}{t-1}\cdot 2\alpha_{t-1}L\right)\\
& +\Big(\frac{t-1}{n}-(k-1)\Big)\Big(\eta_t \mathbb{E}[\delta_{t-1}] + \frac{k}{t-1}\cdot 2\alpha_{t-1}L\Big)\\
&\leq \Big\{\frac{1}{n}\cdot\min\bigl(\eta_t,1\bigr)+\left(1-\frac{1}{n}\right)\cdot\eta_t\Big\} \mathbb{E}[\delta_{t-1}]+\frac{4L \alpha_{t-1}}{n}.
\end{align*}

Finally, combining the above two cases yields the desired result. 
$\hfill \Box$

Before we use Theorem \ref{thm:recursive} to analyze the stability of SGD for convex, strongly convex and non-convex cases respectively, we introduce the following useful lemma 
 which reveals an important advantage of SGD:  it usually takes several steps before the updates $\bw_t$ and $\bw'_t$ of SGD start to differ from each other. 
 
\begin{lemma}\label{lem:interplation}
Assume that the loss function $\ell(\cdot;z,z')$ is nonnegative and $L$-Lipschitz for all pairs $(z,z')$.
Suppose we run SGD for $T$ steps on two samples of size $n$ namely $S$ and $S'$ which differ in at most an example.
Then, for every $t_0 \in \{2,\cdots,n\}$, we have
\begin{eqnarray}\label{eq:lem:interplation}
\mathbb{E}|\ell(\bw_T,z,z')-\ell(\bw'_T,z,z')|\leq \frac{t_0}{n}\sup_{\bw,z,z'}\ell(\bw,z,z')+L \mathbb{E}[\delta_T|\delta_{t_0}=0],
\end{eqnarray}
where $\delta_{t_0}=\|\bw_{t_0}-\bw'_{t_0}\|.$
\end{lemma}

\begin{proof}
Let $z,z' \in Z$ be an arbitrary pair of examples. By the conditional expectation formula and the Lipschitz assumption of $\ell$, we have
\begin{eqnarray*}
&&\mathbb{E}|\ell(\bw_T,z,z')-\ell(\bw'_T,z,z')|\nonumber\\
&&= \mathbb{P}\{\delta_{t_0}\neq0\}\mathbb{E}[|\ell(\bw_T,z,z')-\ell(\bw'_T,z,z')|\big|\delta_{t_0}\neq0]\nonumber\\
&&+\mathbb{P}\{\delta_{t_0}=0\}\mathbb{E}[|\ell(\bw_T,z,z')-\ell(\bw'_T,z,z')|\big|\delta_{t_0}=0]\nonumber\\
&&\leq \mathbb{P}\{\delta_{t_0}\neq 0\}\cdot \sup_{\bw,z,z'}\ell(\bw,z,z')+L \mathbb{E}[\delta_T\big|\delta_{t_0}=0].\nonumber
\end{eqnarray*}
Now we bound $\mathbb{P}\{\delta_{t_0}\neq0\}$ under random permutation and selection rules. 

Under the random permutation rule, denote $t_1^{\star}=\{t\ |\ z_{\xi_t}\neq z'_{\xi_t}, 1\leq t \leq n\}$.
We have 
\begin{equation}\label{eq1:ProofLemmaInter}
\mathbb{P}\{\delta_{t_0}\neq 0\} \leq \mathbb{P}\{t_1^{\star}\leq t_0\}=\frac{t_0}{n}
\end{equation}
since if $t_1^{\star} > t_0$, then we must have $\delta_{t_0} = 0$. 

For the case of random selection rule, let $t^{\star}$ be the first time our algorithm encountering the different examples.
For the same reason behind (\ref{eq1:ProofLemmaInter}), we just need to bound $\mathbb{P}\{t^{\star} \leq t_0\}$ and 
we have
\begin{equation}
\mathbb{P}\{\delta_{t_0}\neq 0\} \leq \mathbb{P}\{t^{\star} \leq t_0\}\leq \sum_{t=1}^{t_0}\mathbb{P}\{t^{\star}= t\} =\frac{t_0}{n}.\nonumber
\end{equation}
Combining these two cases, we complete the proof.
\end{proof}

\subsection{Convex case}
 We present below the first stability result of SGD provided that  the pairwise loss $\ell(\cdot,z,z')$ is convex and strongly smooth.
\begin{theorem}\label{thm:convex}
Assume that the loss function $\ell(\cdot;z,z')$ is $\beta$-smooth, convex and $L$-Lipschitz for every example points $z$ and $z'$.
Suppose that we run SGD with step sizes $\alpha_t\leq2/\beta$ for $T$ steps.
Then, 
\begin{equation}\label{eq:convex-stability}
\epsilon_{stab}(A^{\text{last}},T, \ell, D,n)\leq \frac{4L^2}{n}\sum_{t=1}^{T-1} \alpha_t, 
\end{equation}
and 
\begin{equation}\label{eq:convex-stability2}
\epsilon_{stab}(A^{\text{avg}}, T,\ell, D,n)\leq \frac{4L^2}{T n}\sum_{t=2}^{T}\sum_{j=1}^{t-1} \alpha_j, 
\end{equation}
\end{theorem}
\begin{proof}
We now fix a pair of examples $z$ and $z'$ and apply the Lipschitz condition on $\ell(\cdot,z,z')$ to get
\begin{equation}\label{eq:thm3-1}
\mathbb{E}|\ell(\bw_{T},z,z')-\ell(\bw'_{T},z,z')|\leq L\mathbb{E}[\delta_T],
\end{equation}
where $\delta_T=\|\bw_T-\bw'_{T}\|.$
By Lemma \ref{lem:expansive} and Theorem \ref{thm:recursive}, we have 
$
\mathbb{E}[\delta_{t}]
\leq\mathbb{E}[\delta_{t-1}] + \frac{4L}{n}\cdot\alpha_{t-1}.$
Unraveling the recursion yields \begin{equation}\label{eq:3-2-inter}
\mathbb{E}[\delta_{T}]\leq \frac{4L}{n}\sum_{t=1}^{T-1}\alpha_t.
\end{equation}
Plugging this back into the equation \eqref{eq:thm3-1}, we obtained \eqref{eq:convex-stability}.

To prove \eqref{eq:convex-stability2}, we notice that \eqref{eq:3-2-inter} holds true for any $T$, and therefore
\begin{align*}
&\EX\Bigl[  | \ell(\bar{\bw}_T,z,z') - \ell (\bar{\bw}'_T,z,z') | \Bigr]   
\leq L\mathbb{E}[\|\bar{\bw}_{T}-\bar{\bw}'_{T}\|]\\
&\leq L\cdot\frac{1}{T}\sum_{t=1}^T \mathbb{E}[\|\bw_t-\bw'_t\|] = \frac{L}{T} \sum_{t=1}^T \EX[\gd_t] \le \frac{4L^2}{n T} \sum_{t=2}^T \sum_{j=1}^{t-1}\ga_j, \numberthis
\end{align*}
where we used the fact $\bw_1 = \bw'_1.$ 
This completes the proof of the theorem. 
\end{proof}
If we choose $\ga_t = \frac{2}{\gb t^a}$ with $a \in (0,1)$, then Theorem \ref{thm:convex} tells us that  stability bounds of SGD for pairwise learning schemes are of  order $O( \frac{T^{1-a}}{n})$. If the iteration of SGD is linear with respect to the size of the training data, e.g. $T = n$,  SGD for pairwise learning will achieve stability and generalization error of order $O(\frac{1}{T^a}).$ In this sense, faster training SGD will lead to reasonably good generalization.

\subsection{Strongly convex case}
If, furthermore, the function $\ell$ is strongly convex, we can establish stronger results. 
\begin{theorem}\label{thm:strong-convex}
Assume that the loss function $\ell(\cdot,z,z')$ is $\gamma$-strongly convex, $\beta$-smooth and $L$-Lipschitz for every example points $z$ and $z'$.
Suppose that we run SGD with the constant step size $\alpha\leq\frac{2}{\beta+\gamma}$ for $T$ steps.
Then, SGD satisfies uniform stability with
\begin{equation}\label{eq:scvx-1}\epsilon_{stab}(A^{\text{last}},T, \ell, D,n)\leq \frac{8L^2}{\gamma n}\left[1-(1-\frac{\alpha\gamma}{2})^{T-1}\right].\end{equation}
and 
\begin{equation}\label{eq:scvx-2}\epsilon_{stab}(A^{\text{avg}},T, \ell, D,n)\leq \frac{8L^2}{\gamma T n}\sum_{t=2}^{T}\left[1-(1-\frac{\alpha\gamma}{2})^{t-1}\right].\end{equation}

\end{theorem}
\begin{proof}
Fix a pair of examples $z$ and $z'$ and apply the boundedness of the gradient of $\ell(\cdot,z,z')$ to get
\begin{equation}\label{eq11}
\mathbb{E}|\ell(\bw_{T},z,z')-\ell(\bw'_{T},z,z')|\leq L\mathbb{E}[\delta_T],
\end{equation}where $\delta_T=\|\bw_T-\bw'_{T}\|.$
We then use the recursive relation between $\delta_t$ and $\delta_{t-1}$ as established in Theorem \ref{thm:recursive} to bound  $\delta_T.$
Since $\alpha\leq\frac{2}{\beta+\gamma}$ by assumption, we have $G_t$ is $\left(1-\frac{\beta\gamma\alpha}{\beta+\gamma}\right)$-expansive by Lemma \ref{lem:expansive}.
Moreover we have $1-\frac{\beta\gamma\alpha}{\beta+\gamma}\leq 1-\frac{\alpha\gamma}{2}$ following from $\beta\geq\gamma$ by the definitions.
As a result we have $G_t$ is $\left(1-\frac{\alpha\gamma}{2}\right)$-expansive. Hence $\eta=1-\frac{\alpha\gamma}{2}\in(0,1)$.
Then by Theorem \ref{thm:recursive}, we have
$
\mathbb{E}[\delta_{t}]
\leq \eta\mathbb{E}[\delta_{t-1}] + \frac{4L}{n}\cdot\alpha.
$
Unravel the recursion and we have 
\begin{equation}\label{eq:scvx-inter}
\mathbb{E}[\delta_{T}]\leq \frac{4L\alpha}{n}\sum_{j=0}^{T-2}\eta^j\leq \frac{8L}{\gamma n}(1-\eta^{T-1}).
\end{equation}
Plugging this back into the equation \eqref{eq11} yields \eqref{eq:scvx-1}. 

To prove \eqref{eq:scvx-2}, notice that \eqref{eq:scvx-inter} holds true for any $T$. Consequently,  
\begin{align*}&\EX\Bigl[  | \ell(\bar{\bw}_T,z,z') - \ell (\bar{\bw}'_T,z,z') | \Bigr]  
\leq L\mathbb{E}[\|\bar{\bw}_{T}-\bar{\bw}'_{T}\|]\\
&\leq L\cdot\frac{1}{T}\sum_{t=1}^T \mathbb{E}[\|\bw_t-\bw'_t\|] = \frac{L}{T} \sum_{t=1}^T  \EX [\gd_t] \le \frac{8L^2}{\gamma Tn} \sum_{t=2}^T (1-\eta^{t-1}),\end{align*}
where we used the fact that $\gd_1=0$. 
This completes the proof of the theorem. 
\end{proof}

Theorem \ref{thm:strong-convex} indicates  that, in the strongly convex case, although the uniform stability is also increasing w.r.t. $T$, it is upper bounded by a finite bound, i.e. $\frac{8L^2}{\gamma n}$ which is independent of the running time $T$. 

Note that Theorem \ref{thm:strong-convex} only analyzes the uniform stability of SGD with constant step size which is not commonly used in practice. With the help of Lemma \ref{lem:interplation}, we can establish the following theorem on the stability of a more popular form of SGD when step sizes are  ``staircase" decaying.

\begin{theorem}\label{thm:strong-convex-varying}
Assume that the loss function $\ell(\cdot,z,z')$ is $\gamma$-strongly convex, $\beta$-smooth and $L$-Lipschitz for every example points $z$ and $z'$ and $\rho=\sup_{\bw,z,z'}\ell(\bw,z,z')$.
Let $\lceil\beta/\gamma\rceil$ be the smallest positive integer which is larger than or equals to  $\beta/\gamma$.
Suppose that we run SGD with the varying step sizes $\alpha_t = \frac{2}{\gamma t}$ for $t=1,\ldots,T$ and $T\geq\lceil\beta/\gamma\rceil+1$.
Then, $$\epsilon_{stab} (A^{\text{last}},T, \ell, D, n)\leq \frac{8L^2}{\gamma n}\left(1-\frac{\lceil\beta/\gamma\rceil}{T-1}\right)+\frac{\rho}{n}\left(1+\lceil\beta/\gamma\rceil\right).$$
\end{theorem}
\begin{proof}
It is easy to check that $\alpha_t\leq \frac{2}{\beta+\gamma}$ when $t\geq 1+\frac{\beta}{\gamma}$.
Thus if $t\geq t_0:=1+\lceil\frac{\beta}{\gamma}\rceil$, we have $G_t$ is $\eta_t$-expansive with $\eta_t=1-\frac{1}{t-1}$ 
by Lemma \ref{lem:expansive} and the fact $1-\frac{\beta \gamma}{\beta+\gamma}\cdot \frac{2}{\gamma (t-1)}\leq 1-\frac{1}{t-1}$.
To this end, recalling Lemma \ref{lem:interplation}, we have
\begin{eqnarray}\label{eq:strong-convex:interplation}
\mathbb{E}|\ell(\bw_T,z,z')-\ell(\bw'_T,z,z')|\leq \frac{\rho}{n}\left(1+\left\lceil\frac{\beta}{\gamma}\right\rceil\right)+L \mathbb{E}[\delta_T|\delta_{t_0}=0].
\end{eqnarray}
Next we will bound $\Delta_T:=\mathbb{E}[\delta_T|\delta_{t_0}=0]$.
By Theorem \ref{thm:recursive}, we have
$
\Delta_{t}
\leq (1-\frac{1}{t-1})\Delta_{t-1} + \frac{4L}{n}\cdot\alpha_{t-1}
$
for $t_0\leq t\leq T$.
Unravel the recursion from $t=T$ to $t=t_0$ and we have $
\Delta_{T}\leq \frac{8L}{n\gamma}\cdot\frac{T-t_0}{T-1}.$ 
Plugging this back into the equation \eqref{eq:strong-convex:interplation} yields the desired result.
\end{proof}
For the ``staircase" decaying step sizes, it remains a question to us on how to get similar stability results when the output of SGD is  the average of iterates, i.e. $\ASA.$

\subsection{Non-convex case}
If $\ell(\cdot,z,z')$ is not convex such as in the case of MEE principle \citep{Hu2013,hu2016convergence,principe}, we have the following result. 

\begin{theorem}\label{thm:non-convex}
Assume that the loss function $\ell(\cdot,z,z')\in[0,1]$ is $\beta$-smooth and $L$-Lipschitz for any $z$ and $z'$.
Suppose that we run SGD with step sizes $\alpha_t\leq\frac{c}{t}$ for $T$ steps.
Then, we have 
\begin{equation}
\epsilon_{stab} (A^{\text{last}}, T, \ell, D, n)\leq \frac{1+1/(\beta c)}{n-1}(4cL^2)^{\frac{1}{1+\beta c}}(T-1)^{\frac{\beta c}{1+\beta c}}.\nonumber
\end{equation}
\end{theorem}

\begin{proof}
Firstly, by Lemma \ref{lem:interplation}, we have for every $t_0 \in \{2,\cdots,n\}$,
\begin{equation}\label{eq9-26}
\mathbb{E}|\ell(\bw_T;z,z')-\ell(\bw'_T;z,z')|\leq \frac{t_0}{n}+L \mathbb{E}[\delta_T|\delta_{t_0}=0].
\end{equation}
Next, we will bound $\Delta_T:=\mathbb{E}[\delta_T|\delta_{t_0}=0]$ as a function of $t_0$ and then minimize for $t_0$.
By Lemma \ref{lem:expansive} and a variant of Theorem \ref{thm:recursive} modified for conditional expectation, we have
$
\Delta_t
\leq\left(1+(1-1/n)\frac{\beta c}{t-1}\right)\Delta_{t-1}+ \frac{4cL}{n(t-1)}
\leq \exp\left\{(1-1/n)\frac{\beta c}{t-1}\right\}\Delta_{t-1}+ \frac{4cL}{n(t-1)}.$
Unwind this recurrence relation from $T$ down to $t_0+1$. This gives 
\begin{eqnarray*}
\Delta_T\leq\frac{4cL}{n(T-1)}+\sum_{t=t_0}^{T-2}\prod_{s=t+1}^{T-1} \exp\left\{(1-1/n)\frac{\beta c}{s}\right\} \frac{4cL}{nt},
\end{eqnarray*}
wherein the second term
\begin{eqnarray*}
&&\sum_{t=t_0}^{T-2}\prod_{s=t+1}^{T-1} \exp\left\{(1-1/n)\frac{\beta c}{s}\right\} \frac{4cL}{nt} \nonumber\\
&&=\frac{4cL}{n}\sum_{t=t_0}^{T-2} \Big\{\exp\Big[(1-1/n)c\beta \sum_{s=t+1}^{T-1}\frac{1}{s}\Big]\Big\}\frac{1}{t} \nonumber\\
&&\leq\frac{4cL}{n}\sum_{t=t_0}^{T-2} \Big\{\exp\Big[(1-1/n)c\beta \ln\Big(\frac{T-1}{t}\Big)\Big]\Big\}\frac{1}{t} \nonumber\\
&&\leq\frac{4cL}{n}(T-1)^{(1-1/n)c\beta} \sum_{t=t_0}^{T-2} t^{-(1-1/n)c\beta-1}\leq \frac{4L}{(n-1)\beta}\Big(\frac{T-1}{t_0-1}\Big)^{c\beta}.
\end{eqnarray*}
Thus we can omit the higher order infinitesimals term, i.e. $\frac{4cL}{n(T-1)}$, and we have
$
\Delta_T\leq\frac{4L}{(n-1)\beta}\left(\frac{T-1}{t_0-1}\right)^{c\beta}.
$
Plugging this bound into (\ref{eq9-26}), we have
$$
\mathbb{E}|\ell(\bw_T,z,z')-\ell(\bw'_T,z,z')|\leq \frac{t_0}{n}+\frac{4L^2}{(n-1)\beta}\left(\frac{T-1}{t_0-1}\right)^{c\beta}.
$$
The right hand side is approximately minimized when 
$
t_0=(4cL^2)^{\frac{1}{1+\beta c}}(T-1)^{\frac{\beta c}{1+\beta c}}+1.
$
Omitting the higher order infinitesimals, we have
\begin{align*}
\mathbb{E}|\ell(\bw_{T},z,z')-\ell(\bw'_{T},z,z')|\leq \frac{1+1/(\beta c)}{n-1}(4cL^2)^{\frac{1}{1+\beta c}}(T-1)^{\frac{\beta c}{1+\beta c}},
\end{align*}
 and we complete the proof.
\end{proof}
For the non-convex case, it also remains a question to us on how to get similar stability results when the output of SGD is  the average of iterates, i.e., $\ASA.$ Note that Lemma \ref{lem:interplation} plays a key role in the stability analysis of SGD in the general non-convex case, where the gradient updates $G_t$ are no longer non-expansive operations  in contrast to the convex case using  Lemma \ref{lem:expansive}.  

We end Section \ref{sec:stability-analysis} with a useful remark.  The stability results above also hold true for 
the projected SGD algorithm defined by 
\begin{equation}\label{eq:proj-sgd}
\bw_{t}=\Pi_{\Omega}\Big\{\bw_{t-1}-\frac{\alpha_{t -1}}{t-1}\ \sum_{j=1}^{t-1}\nabla\ell(\bw_{t-1}, z_{\xi_t},z_{\xi_j})\Big\},
\end{equation}
where $\gO$ is a bounded convex domain in $\R^d$,  and $\Pi_{\Omega}$ is the projection operator defined by $\Pi_\gO(\bu)=\argmin_{\bw\in \Omega}\left\|\bu-\bw\right\|$.
Typically, one can choose $\gO$ to be a bounded ball with center zero, i.e.  $\Omega=\{\bw: \|\bw\|\leq r_0\}$ for which the projection operator can be computed analytically.
In this case, the projection onto a convex set is a non-expansive operation, i.e.
\begin{eqnarray*}
\delta_t =\|\Pi_{\Omega}(G_t(\bw_{t-1})-G'_t(\bw'_{t-1}))\|\leq\|G_t(\bw_{t-1})-G'_t(\bw'_{t-1})\|.
\end{eqnarray*}
As a result, our previous proof techniques in the case of the original (non-projected) SGD algorithm defined by \eqref{eq:sgd} can still be applied to this situation.
Consequently, the stability results stated in the above theorems hold true for the projected SGD.

\section{Trade-off between Stability and Optimization Error}\label{sec:trade-off}
 In this section, we will start from the trade-off connection in Lemma \ref{lem:trade-off}  to establish the minimax lower bound for the excess expected risk. Then, we will combine this with the stability results in Section \ref{sec:stability-analysis} to derive the lower bounds for the optimization error of SGD algorithms in the setting of pairwise learning. 
 
 \subsection{Minimax Lower Bounds}
In particular, let $\gO$ be a bounded convex domain with finite diameter, i.e. $|\gO| <\infty$. We consider the class $\mathcal{L}_{\text{c}}$ of convex and strongly smooth pairwise losses which is defined by  \begin{align*}
\mathcal{L}_{\text{c}}=\{\ell:\Omega\times\mathcal{Z}\times\mathcal{Z}\rightarrow\mathbb{R}\ |\ \ell \ \text{is convex},\beta-\text{smooth};  |\Omega|<\infty \},
\end{align*}
and the class of strongly convex and smooth pairwise losses which is given by \begin{align*}
\mathcal{L}_{\text{sc}}=\{\ell:\Omega\times\mathcal{Z}\times\mathcal{Z}\rightarrow\mathbb{R}\ |\ \ell \ \text{is} \ \gamma-\text{strongly convex},\beta-\text{smooth}; |\Omega|<\infty \}.
\end{align*}

For the class $\mathcal{L}_{\text{c}}$ of pairwise loss functions, we have the following lower bound for the minimax risk. 
\begin{theorem}\label{thm:convex-minimax}There exists a pairwise loss $\ell\in \mathcal{L}_{\text{c}}$ such that 
\begin{equation}
\inf_{\hbw}\sup_{D\in\mathcal{D}}\mathbb{E}_{S\sim D^n} [\Delta R(\hbw (S))] \geq \frac{3\beta |\Omega|^2}{ 128\sqrt{6n}}.
\end{equation}
 \end{theorem}

The proof of Theorem \ref{thm:convex-minimax} can be found in \ref{App:Proof1} which involves the Le Cam's method \cite{le2012asymptotic}.

An immediate by-product result from the above theorem is the following corollary which states the lower bound for the excess expected risk when $\ell\in \mathcal{L}_{\text{c}}$. 
\begin{corollary} There holds 
	\begin{equation}\label{eq:byproduct}
 	\inf_{\hbw}\sup_{\ell \in \mathcal{L}_{\text{c}}, D\in\mathcal{D}}\mathbb{E}_{S} [\Delta R(\hbw)] \ge  \frac{3\beta |\Omega|^2}{ 128\sqrt{6n}}.
 \end{equation}
 \end{corollary}
 \begin{proof} The result follows directly from Theorem \ref{thm:convex-minimax}  and the elementary inequality:  $$\inf_{\hbw}\sup_{\ell \in \mathcal{L}_{\text{c}}, D\in\mathcal{D}}\mathbb{E}_{S} [\Delta R(\hbw)] \geq \sup_{\ell \in \mathcal{L}_{\text{c}}}\inf_{\hbw}\sup_{D\in\mathcal{D}}\mathbb{E}_{S} [\Delta R(\hbw). $$
 \end{proof}
We now present the lower bound for the minimax risk for the class  $\L_{\text{sc}}$ of pairwise losses. 
\begin{theorem}\label{thm:sconvex-minimax}
There exists a  pairwise loss $\ell\in \L_{sc}$ such that 
\begin{equation}
\inf_{\hbw}\sup_{D\in\mathcal{D}}\mathbb{E}_{S} [\Delta R( \hbw )] \geq \frac{\beta |\Omega|^2}{32 n}.
\end{equation}
\end{theorem}
We postpone the proof of Theorem \ref{thm:sconvex-minimax} to \ref{App:Proof2}.  An immediate result is the following lower bound for the excess expected risk for $\ell\in \mathcal{L}_{\text{sc}}.$   \begin{corollary} There holds 
	\begin{equation}\label{eq:byproduct-2}
 	\inf_{\hbw}\sup_{\ell \in \mathcal{L}_{sc}, D\in\mathcal{D}}\mathbb{E}_{S} [\Delta R(\hbw)] \ge  \frac{\beta |\Omega|^2}{32 n}.
 \end{equation}
 \end{corollary}
 \begin{proof} The result follows directly from Theorem \ref{thm:sconvex-minimax} and the elementary inequality:  $$\inf_{\hbw}\sup_{\ell \in \mathcal{L}_{sc}, D\in\mathcal{D}}\mathbb{E}_{S} [\Delta R(\hbw)] \geq \sup_{\ell \in \mathcal{L}_{sc}}inf_{\hbw}\sup_{D\in\mathcal{D}}\mathbb{E}_{S} [\Delta R(\hbw)]. $$
 \end{proof}
 

\subsection{Optimization Lower Bounds for SGD of Pairwise Learning}
In this subsection, we assume now that there exists an absolute constant $b>0$ such that,  for any loss $\ell\in \mathcal{L}$ where $\mathcal{L}$ can be $\mathcal{L}_{\text{c}}$ or $\mathcal{L}_{\text{sc}}$ for different settings in our consideration, there holds 
$$\sup_{z,z'\in\Z}  \min_{\bw\in \gO}\|\nabla \ell(\bw,z,z')\| \le b.$$  Under this condition, we can see that $\ell$ is ($ |\gO| \gb + b$)-Lipschitz.  Indeed,  for any fixed $z, z'\in \Z$, assume $\bw_0 = \arg\min_{\bw\in \gO} \| \nabla \ell(\bw, z,z')\|$. Then, by the $\gb$-smoothness of $\ell$, we have, for any $\bw\in \gO$, that  
$ \|\nabla \ell(\bw,z,z') - \nabla \ell(\bw_0,z,z')\| \le \gb \| \bw-\bw_0
\| \le \gb |\gO|.$ This indicates that $ \|\nabla \ell(\bw,z,z') \| \le \|\nabla \ell(\bw,z,z') - \nabla \ell(\bw_0,z,z')\| + \|\nabla \ell(\bw_0,z,z')\| \le \gb |\gO| + b.$  Since $z,z'$ and $\bw$ are arbitrary, it follows that  $\ell$ is ($ |\gO| \gb + b$)-Lipschitz, i.e., $L=|\gO| \gb + b.$

Combining the minimum lower bound in Theorem \ref{thm:convex-minimax} with Lemma \ref{lem:trade-off}, one can derive the following lower bound for SGD of pairwise learning with  smooth convex loss functions. 
\begin{theorem}\label{thm:convex-convergence-lower-bound}
Consider the output $\ASA$ of the projected SGD with step sizes $\alpha_t$ at iteration $T$ based on a pairwise loss $\ell\in\mathcal{L}_{\text{c}}$, and  the following cases: \begin{enumerate}
    \item \textbf{Constant step size:} $\alpha_t \equiv \alpha=\frac{c}{T^{a}}\leq \frac{2}{\beta}$ with $a \in [0,1)$;
    \item \textbf{Staircase decaying step sizes:} $\alpha_t=\frac{c}{t^{a}}$ with $a \in (0,1)$ and $c\leq \frac{2}{\beta}$.
\end{enumerate}
Then, for either of the above cases, there exists a universal constant $\widetilde{C}_1$, and $T_0$ such that, for any $T\ge T_0$, there holds 
$\OAVG(T,\mathcal{L}_c,\mathcal{D},n) \ge \frac{\widetilde{C}_1}{T^{1-a}}.$
\end{theorem}

\begin{proof}
1)~
Putting  $\alpha_t \equiv \frac{c}{T^{a}}$ back into \eqref{eq:convex-stability2} in Theorem \ref{thm:convex} implies that 
\begin{equation}\label{eq:convex-stability-constant-step}
 \SAVG(T,\mathcal{L}_c,\mathcal{D},n)\leq \frac{4L^2}{nT} \sum_{t=2}^T \sum_{j=1}^{t-1} \ga_j  \le \frac{4 cL^2}{nT^{1+a}}\sum_{t=2}^{T}(t-1)\leq \frac{4cL^2}{n} \cdot T^{1-a}.   
\end{equation}
Noting the relation \eqref{eq3-2-tradeoff} and applying Theorem \ref{thm:convex-minimax}, we have 
$$
\frac{8cL^2}{n} \cdot T^{1-a}+\OAVG(T,\mathcal{L}_c,\mathcal{D},n)\geq \frac{3\beta |\Omega|^2}{ 128\sqrt{6n}}.
$$
It follows that 
$$
\OAVG(T,\mathcal{L}_c,\mathcal{D},n)\geq \frac{3\beta |\Omega|^2}{ 128\sqrt{6n}}-\frac{8cL^2}{n} \cdot T^{1-a}:=Q(n).
$$
Note that it is well known that the optimization error of the projected SGD is independent of  the sample size $n$ (see the results in \cite{Kar} for example).
That means  $\OAVG(T,\mathcal{L}_c,\mathcal{D},n)$ is actually not a function of $n$ although we include $n$ in its construction. 
As a result, we can take maximum of $Q(n)$ over $n$ so that the resulting lower bound is ``best".
To this end, letting $\tau_0=\left[\frac{3\beta |\Omega|^2}{2048\sqrt{6}cL^2}\right]^{1/(1-a)}$ and $C_0=\frac{3\beta^2 |\Omega|^4}{1048576cL^2}$, we can rewrite $Q(n)$ as  $$Q(n)=\frac{C_0}{T^{1-a}}-8cL^2T^{1-a}\left[\frac{1}{\sqrt{n}}-(\frac{\tau_0}{T})^{1-a}\right]^2.$$
Thus for sufficiently large $T\geq \tau_0$, we can always find an integer $n_0$ such that  
$\frac{2}{3}\left(\frac{T}{\tau_0}\right)^{1-a}\leq \sqrt{n_0} \leq 2\left(\frac{T}{\tau_0}\right)^{1-a}$.
Let $C_1=\frac{9\beta^2 |\Omega|^4}{4194304cL^2}.$
As a result, we have $$Q(n_0)\geq \frac{C_0}{T^{1-a}}-2cL^2T^{1-a}\left(\frac{\tau_0}{T}\right)^{2(1-a)}=\frac{C_1}{T^{1-a}}.$$ 
Thus we obtain, for any $T\ge \tau_0$
$$
\mathcal{E}_{\text{opt}}^{\text{avg}}(T,\mathcal{L}_c,\mathcal{D},n_0)\geq\frac{C_1}{T^{1-a}}.
$$  Since $ \mathcal{E}_{\text{opt}}^{\text{avg}}(T,\mathcal{L}_c,\mathcal{D},n)$ is independent of $n$, we establish the desired result.  
\medskip 

2)~ Plug $\alpha_t =\frac{c}{t^{a}}$ into \eqref{eq:convex-stability2} and let $c'=c/(1-a)$. We have 
\begin{equation}\label{eq:convex-stability-varying-step}
 \SAVG(T,\mathcal{L}_c,\mathcal{D},n)\leq \frac{4L^2}{Tn}\sum_{t=2}^{T}\sum_{j=1}^{t-1}\frac{c}{j^a}\leq \frac{4cL^2}{(1-a)n}\sum_{t=2}^{T} \frac{t^{1-a}}{T}\leq\frac{4c'L^2}{n} \cdot T^{1-a}. 
\end{equation}
Recall the first equation namely \eqref{eq:convex-stability-constant-step} in the proof of the first case. 
We can find that the only difference between these two stability results, namely \eqref{eq:convex-stability-varying-step} and \eqref{eq:convex-stability-constant-step}, comes from we replacing $c$ by $c'$. Likewise denote $\tau'_0=\left[\frac{3\beta |\Omega|^2}{2048\sqrt{6}c'L^2}\right]^{1/(1-a)}$ and $C'_0=\frac{3\beta^2 |\Omega|^4}{1048576c'L^2}$.
Thus for sufficiently large $T\geq \tau'_0$, we can always find an integer $n'_0$ s.t. 
$\frac{2}{3}\left(\frac{T}{\tau'_0}\right)^{1-a}\leq \sqrt{n'_0} \leq 2\left(\frac{T}{\tau'_0}\right)^{1-a}$.
Let $C'_1=\frac{9\beta^2 |\Omega|^4}{4194304c'L^2}.$
It is natural to use the same strategy to obtain almost the same lower bound for the optimization error as the case of constant step size, viz.,
$$
\OAVG(T,\mathcal{L}_c,\mathcal{D},n'_0)\geq\frac{C'_1}{T^{1-a}}.
$$
Again, as $ \mathcal{E}_{\text{opt}}^{\text{avg}}(T,\mathcal{L}_c,\mathcal{D},n)$ is independent of $n$, the desired result is proved. 
\end{proof}
The work of  \cite{Wang} considered the regret bound for projected online gradient descent algorithm in pairwise learning which is exactly SGD algorithm we consider here in the stochastic setting.  Specifically, in \cite[Theorem 13]{Wang}, they gave the regret rate of $O(\sqrt{T})$ of the projected online gradient descent algorithm with varying step sizes $\alpha_t=O\Bigl(\frac{1}{\sqrt{t}}\Bigr)$ for pairwise learning.
While in \cite[Theorem 3]{Kar}, they obtained the online to batch conversion bound. Combining the above results, we can obtain an upper bound of the convergence rate, i.e., $O\Bigl(\frac{1}{\sqrt{T}}\Bigr)$ (up to a $\log{T}$ factor). This result meets the lower bound we have established in Theorem \ref{thm:convex-convergence-lower-bound} which says this algorithm can not have better worst-case convergence rate than $O\Bigl(\frac{1}{\sqrt{T}}\Bigr)$ with step sizes $\alpha_t=O\Bigl(\frac{1}{\sqrt{t}}\Bigr)$ in the general convex smooth case. Thus our results confirm its optimality up to a logarithmic factor.

From Theorem \ref{thm:sconvex-minimax}, we can get the following two theorems about the lower bounds for the optimization error of SGD with fixed step size and varying step sizes respectively in the setting of smooth strongly convex loss functions. 
\begin{theorem}\label{thm:sconvex-convergence-lower-bound-1}
Let the projected SGD with fixed step size $\alpha_t\equiv \ga \leq \frac{2}{\beta+\gamma}$ for $T$ iterations to get an output $\ASA$ based on a pairwise loss $\ell\in\mathcal{L}_{\text{sc}}$.
Then we can get the following results, viz., 
$$\mathcal{E}_{\text{opt}}^{\text{avg}}(T,\mathcal{L}_{sc},\mathcal{D},n) \geq \frac{16(|\Omega|\beta+b)^2}{\gamma n}\left(1-\frac{\alpha\gamma}{2}\right)^{T-1}-C,$$
wherein the offset $C=\frac{1}{n}\left(\frac{16(|\Omega|\beta+b)^2}{\gamma }-\frac{\beta |\Omega|^2}{32 }\right)>0.$
\end{theorem}
\begin{proof}
 Recall \eqref{eq:scvx-2}. We have
\begin{align*}\label{eq:sc-averaged-w}
    &\mathbb{E}|\ell(\bar{\bw}_{T},z,z')-\ell(\bar{\bw}'_{T},z,z')|\leq L\mathbb{E}[\|\bar{\bw}_{T}-\bar{\bw}'_{T}\|]\\
    & \leq \frac{8L^2}{\gamma n} \frac{1}{T} \sum_{t=2}^T (1-\eta^{t-1})\leq \frac{8L^2}{\gamma n}\left[1-\eta^{T-1}\right].\numberthis
\end{align*}
 Thus we have $$\mathcal{E}_{\text{stab}}^{\text{avg}}(T,\mathcal{L}_{sc},\mathcal{D},n)\leq \frac{8L^2}{\gamma n}\left[1-(1-\frac{\alpha\gamma}{2})^{T-1}\right].$$
Noting the relation \eqref{eq3-2-tradeoff2} and applying Theorem \ref{thm:sconvex-minimax}, we have $$ \frac{16L^2}{\gamma n}\left[1-(1-\frac{\alpha\gamma}{2})^{T-1}\right] +\mathcal{E}_{\text{opt}}^{\text{avg}}(T,\mathcal{L}_c,\mathcal{D},n)\geq \frac{\beta |\Omega|^2}{32 n}.$$
It follows that $$\mathcal{E}_{\text{opt}}^{\text{avg}}(T,\mathcal{L}_c,\mathcal{D},n)\geq \frac{\beta |\Omega|^2}{32 n}-\frac{16L^2}{\gamma n}\left[1-(1-\frac{\alpha\gamma}{2})^{T-1}\right].$$
Recall $L=|\gO|\gb+b$ and we have finished the proof.
\end{proof}

\begin{theorem}\label{thm:sconvex-convergence-lower-bound-2}
Let the projected SGD with step sizes $\alpha_t$ for $T$ iterations to get an averaged output $\ASA$ based on a pairwise loss $\ell\in\mathcal{L}_{\text{sc}}$.
Let $\alpha_t=\frac{2}{\gamma t}$.   
Denote $C:=\frac{2(\gb|\gO|^2+b|\gO|)}{n}\cdot\left(\frac{\beta}{\gamma}+3\right) - \frac{\beta |\Omega|^2}{32 n} + \frac{16(|\gO|\gb+b)^2}{n\gamma}\cdot  \ln\left(\frac{\beta}{\gamma}+3\right)$.
Then, $$\mathcal{E}_{\text{opt}}^{\text{avg}}(T,\mathcal{L}_{sc},\mathcal{D},n) \geq \frac{16L^2(\beta+\gamma)}{\gamma^2 n}\cdot\frac{\ln T}{T}-C.$$ 
\end{theorem}
\begin{proof}
It is easy to check that $\alpha_t\leq \frac{2}{\beta+\gamma}$ when $t\geq 1+\frac{\beta}{\gamma}$. Let $\lceil\beta/\gamma\rceil$ be the smallest positive integer which is larger than or equals to  $\beta/\gamma$.
Thus if $t\geq t_0:=2+\lceil\frac{\beta}{\gamma}\rceil$, we have $G_t$ is $\eta_t$-expansive with $\eta_t=1-\frac{1}{t-1}$ due to Lemma \ref{lem:expansive} and the fact $1-\frac{\beta \gamma}{\beta+\gamma}\cdot \frac{2}{\gamma (t-1)}\leq 1-\frac{1}{t-1}$.

Let $\delta_t=\|\bw_t-\bw'_t\|$ and  $t_1^{\star}$ be the first time that the SGD algorithms encounter the different examples. 
By the conditional expectation formula, we have
\begin{eqnarray*}
&&\mathbb{E}[\delta_t]= \mathbb{P}\{t_1^{\star} \leq t_0\}\mathbb{E}[\delta_t\big|t_1^{\star} \leq t_0]+\mathbb{P}\{t_1^{\star} > t_0\}\mathbb{E}[\delta_t\big|t_1^{\star} > t_0]\nonumber\\
&&= \frac{t_0}{n}\cdot \mathbb{E}[\delta_t\big|t_1^{\star} \leq t_0]  + \left(1-\frac{t_0}{n}\right)\cdot \mathbb{E}[\delta_t\big|t_1^{\star} > t_0].\nonumber
\end{eqnarray*}
If $t< t_0$, we have $\mathbb{E}[\delta_t\big|t_1^{\star} > t_0]=0$ as the SGD algorithms have not encountered the different examples during the first $t$ steps. Thus when $t< t_0$ we have
\begin{align}\label{eq:th4-4-delta1}
\mathbb{E}[\delta_t]= \frac{t_0}{n}\cdot \mathbb{E}[\delta_t\big|t_1^{\star} \leq t_0]\leq \frac{t_0}{n}\cdot|\Omega|.	
\end{align}
If $t\geq t_0$, we have $$\mathbb{E}[\delta_t]\leq \frac{t_0}{n}\cdot|\Omega| + \mathbb{E}[\delta_t\big|t_1^{\star} > t_0].$$ Denote $\Delta_t:=\mathbb{E}[\delta_t|t_1^{\star} > t_0]$. Recall $G_t$ is $\eta_t$-expansive with $\eta_t=1-\frac{1}{t-1}$.  By Theorem \ref{thm:recursive}, we have
$\Delta_t \leq (1-\frac{1}{t-1})\Delta_{t-1} + \frac{4L}{n}\cdot\alpha_{t-1}$
for $t\geq t_0$. Unravel the recursion from $t$ to $t_0$ and we have $\Delta_{t}\leq \frac{8L}{n\gamma}\cdot\frac{t-t_0}{t-1}.$ Thus when $t\geq t_0$, we have 
\begin{align}\label{eq:th4-4-delta2}
\mathbb{E}[\delta_t]\leq \frac{t_0}{n}\cdot|\Omega| + \frac{8L}{n\gamma}\cdot\frac{t-t_0}{t-1}.	
\end{align}
Combining \eqref{eq:th4-4-delta1} and \eqref{eq:th4-4-delta2}, for $t\geq 2$, we have 
\begin{align}\label{eq:th4-4-delta}
	\mathbb{E}[\delta_t]\leq \frac{t_0}{n}\cdot|\Omega| + \frac{8L}{n\gamma}\cdot\frac{(t-t_0)_{+}}{t-1},
\end{align} where $(t-t_0)_{+}=\max(0,t-t_0)$.

Let $\bar{\bw}_{T}=\frac{1}{T}\sum_{t=1}^T \bw_t$.
Using the Lipschitz condition of $\ell(\cdot;z,z')$, we further have
\begin{align*}\label{eq:sc-averaged-w-1}
    &\mathbb{E}|\ell(\bar{\bw}_{T},z,z')-\ell(\bar{\bw}'_{T},z,z')|\leq L\mathbb{E}[\|\bar{\bw}_{T}-\bar{\bw}'_{T}\|]\\
    &\leq L\cdot\frac{1}{T}\sum_{t=2}^T \mathbb{E}[\|\bw_t-\bw'_t\|]=L\cdot\frac{1}{T}\sum_{t=2}^T\mathbb{E}[\delta_t] \\
    & \leq \frac{t_0}{n}\cdot L|\Omega| + \frac{8L^2}{n\gamma}\cdot \frac{1}{T}\sum_{t=2}^T \frac{(t-t_0)_{+}}{t-1},\numberthis
\end{align*}
wherein the last inequality comes from \eqref{eq:th4-4-delta}.
Next we will bound $\sum_{t=2}^T \frac{(t-t_0)_{+}}{t-1}$. 
Actually we can write 
\begin{align*}\label{eq:sc-averaged-w-2}
    &\sum_{t=2}^T \frac{(t-t_0)_{+}}{t-1} = \sum_{t=t_0+1}^T \frac{t-t_0}{t-1} = \sum_{t=t_0+1}^T \frac{t-1+1-t_0}{t-1}= T-t_0 - \sum_{t=t_0+1}^T \frac{t_0-1}{t-1}\\
    & \leq T-t_0 - \int_{t=t_0+1}^{T+1} \frac{t_0-1}{t-1} dt = T-t_0 -(t_0-1)(\ln T -\ln t_0) \\
    & = T-t_0 + (t_0-1)\cdot\ln t_0 -(t_0-1)\cdot\ln T \\
    & \leq (T-1)\cdot\ln t_0 -(t_0-1)\cdot\ln T,\numberthis
\end{align*}
where the last inequality comes from the fact $\ln t_0 \geq 1$ as $t_0:=2+\lceil\frac{\beta}{\gamma}\rceil\geq 3$.
Substituting \eqref{eq:sc-averaged-w-2} into \eqref{eq:sc-averaged-w-1}, we have $$\mathbb{E}|\ell(\bar{\bw}_{T},z,z')-\ell(\bar{\bw'}_{T},z,z')|\leq  \frac{t_0}{n}\cdot L|\Omega| + \frac{8L^2}{n\gamma}\cdot  \ln t_0 - \frac{8L^2}{n\gamma}\cdot (t_0-1)\cdot\frac{\ln T}{T}. $$
Recall $t_0=2+\lceil\frac{\beta}{\gamma}\rceil$. Thus $2 + \frac{\beta}{\gamma}\leq t_0 \leq 3+\frac{\beta}{\gamma}.$ As a result, we have $$\mathcal{E}^{\text{avg}}_{\text{stab}}(T,\mathcal{L}_{sc},\mathcal{D},n)\leq  \frac{L|\Omega|}{n}\cdot\left(\frac{\beta}{\gamma}+3\right)  + \frac{8L^2}{n\gamma}\cdot  \ln\left(\frac{\beta}{\gamma}+3\right) - \frac{8L^2}{n\gamma}\cdot\left(\frac{\beta}{\gamma}+1\right)\cdot\frac{\ln T}{T}. $$

Noting the relation \eqref{eq3-2-tradeoff2} and applying Theorem \ref{thm:sconvex-minimax}, we have $$ 2 \mathcal{E}^{\text{avg}}_{\text{stab}}(T,\mathcal{L}_{sc},\mathcal{D},n) +\mathcal{E}^{\text{avg}}_{\text{opt}}(T,\mathcal{L}_c,\mathcal{D},n)\geq \frac{\beta |\Omega|^2}{32 n}.$$
It follows that
\begin{align*}
    &\mathcal{E}_{\text{opt}}^{\text{avg}}(T,\mathcal{L}_c,\mathcal{D},n)\nonumber\\
    &\geq \frac{\beta |\Omega|^2}{32 n}-\frac{2L|\Omega|}{n}\cdot\left(\frac{\beta}{\gamma}+3\right)  - \frac{16L^2}{n\gamma}\cdot  \ln\left(\frac{\beta}{\gamma}+3\right) + \frac{16L^2}{n\gamma}\cdot\left(\frac{\beta}{\gamma}+1\right)\cdot\frac{\ln T}{T}\nonumber\\
    & \geq  \frac{16L^2}{n\gamma}\cdot\left(\frac{\beta}{\gamma}+1\right)\cdot\frac{\ln T}{T}- \left\{\frac{2L|\Omega|}{n}\cdot\left(\frac{\beta}{\gamma}+3\right) - \frac{\beta |\Omega|^2}{32 n} + \frac{16L^2}{n\gamma}\cdot  \ln\left(\frac{\beta}{\gamma}+3\right)\right\}.
\end{align*}
Recall that $L=|\gO|\gb+b$ and  $C=\frac{2L|\Omega|}{n}\cdot\left(\frac{\beta}{\gamma}+3\right) - \frac{\beta |\Omega|^2}{32 n} + \frac{16L^2}{n\gamma}\cdot  \ln\left(\frac{\beta}{\gamma}+3\right)$.
We have obtained the desired lower bound.
\end{proof}
To illustrate the practical value of Theorem \ref{thm:sconvex-convergence-lower-bound-2}, we recall the work of \cite{Kar}. In \cite[Theorem 5]{Kar}, they established the first fast convergence rate for averaged outputs of online gradient descent algorithm for strongly convex loss functions. Following a variant of \cite[Theorem 1]{zinkevich2003online} in which we choose the step sizes $\alpha_t=O\Bigl(\frac{1}{t}\Bigr)$, we can get a regret bound of $\log(T)$ for the projected online gradient descent algorithm. Combine these two results and we obtain an upper bound of the optimization error, i.e., $O\Bigl(\frac{\log{T}}{T}\Bigr)$. However, our theory can only obtain a matching lower bound with an undesirable offset $C$.

\section{Examples}\label{sec:example}
In this section, we illustrate the stability results obtained in Section \ref{sec:stability-analysis} using three specific examples, namely, AUC maximization, metric learning and MEE. In the following examples, the model parameter $\bw$ is assumed to be in $\Omega=\{\bw: \|\bw\|\leq r_0\}$. 
In addition, we assume $\|x\|\leq B_1$  and  $|y|\leq B_2.$    

In the following, $\epsilon_{stab}(A, T, \ell, D,n)$ means the stability parameter for both the last output of SGD and the average of its iterates. 

\subsection{AUC Maximization}
Area under ROC (AUC) is a metric which is widely used for measuring the
classification performance for imbalanced data \citep{bradley1997use,fawcett2008prie,Hanley}. The AUC score of a scoring funciton is the probability of a random positive example ranking higher than a random negative example \citep{Hanley,Clem}. 
Here we consider a population version of the regularization framework for AUC maximization in \citep{YWL}: 
\begin{equation}\label{auc-1}
\min_{\bw}R(\bw):=\mathbb{E}[\ell(\bw,z,z')],
\end{equation}
where $\ell(\bw,z,z')=(1-(x-x')^{\top}\bw)^2\mathbb{I}_{\{y=1\land y'=-1\}}+(\mu/2)\|\bw\|^2.$
Note that an optimal solution $\bw^{\star}$ for $R(\bw)$ must lie in a ball about 0 with the radius $r_0=\sqrt{2/\mu}$ since
$(\mu/2)\|\bw^{\star}\|^2 \leq R(\bw^{\star}) \leq R(0) \leq 1.$  Hence one can let $\bw$ in \eqref{auc-1} satisfying  $\|\bw\|\le r_0.$
As an application of  Theorem \ref{thm:strong-convex}, we have 
\begin{corollary}\label{cor:AUC}
For the AUC maximization problem (\ref{auc-1}), the loss function $\ell(\cdot;z,z')$ is $\mu$-strongly convex, $(4B_1+8B_1^2 \sqrt{2/\mu}+ \sqrt{2\mu})$-Lipschitz and $(8B_1^2+\mu)$-smooth for every example points $z$ and $z'$.
The projected SGD  with the constant step size $\alpha\leq1/(4B_1^2+\mu)$ has the stability
$$
\epsilon_{stab}(A, T, \ell, D,n)\leq \frac{8(4B_1+8B_1^2\sqrt{2/\mu} + \sqrt{2\mu})^2}{n\mu} \left[1-(1-\frac{\alpha\mu}{2})^{T-1}\right].
$$
\end{corollary}
\begin{proof}
Since
$\ell(\bw;z,z')=(1-(x-x')^{\top}\bw)^2\mathbb{I}_{\{y=1\land y'=-1\}}+(\mu/2)\|\bw\|^2$,
it is easy to check that $\ell(\bw;z,z')$ is
$(4B_1+8B_1^2 r_0 + \mu r_0)$-Lipschitz and $(8B_1^2+\mu)$-smooth for every example points $z$ and $z'$.
 Note that $r_0=\sqrt{2/\mu}$. 
Then we finish the proof by substituting these constants into Theorem \ref{thm:strong-convex}.
\end{proof}
For the case of varying step sizes, applying  Theorem \ref{thm:strong-convex-varying}, we have 
\begin{corollary}\label{cor:AUC-varying}
For the AUC maximization problem (\ref{auc-1}), the loss function $\ell(\cdot;z,z')$ is $\mu$-strongly convex, $(4B_1+8B_1^2 \sqrt{2/\mu}+ \sqrt{2\mu})$-Lipschitz and $(8B_1^2+\mu)$-smooth for every example points $z$ and $z'$.
The projected SGD  with the constant step size $\alpha_t=\frac{2}{\gamma t}$ has the stability
\begin{align*}
\epsilon_{stab}(A^{\text{last}}, T, \ell, D,n) & \leq \frac{8(4B_1+8B_1^2 \sqrt{2/\mu}+ \sqrt{2\mu})^2}{n\mu} \\  & \times \left(1-\frac{1+\lceil8B_1^2/\mu\rceil}{T-1}\right)+\frac{\rho}{n}\left(2+\lceil8B_1^2/\mu\rceil\right),
\end{align*}
wherein $\rho=1+(1+2B_1 \sqrt{2/\mu})^2$.
\end{corollary}
\begin{proof}
We just need to show that $\sup_{\bw,z,z'}\ell(\bw,z,z')=1+(1+2B_1 \sqrt{2/\mu})^2$.
Since $\ell(\bw;z,z')=(1-(x-x')^{\top}\bw)^2\mathbb{I}_{\{y=1\land y'=-1\}}+(\mu/2)\|\bw\|^2$ and $\|x\|\leq B_1$, $\|\bw\|\leq r_0$ by assumption, it is direct to find that 
$\sup_{\bw,z,z'}\ell(\bw,z,z')\leq \frac{\mu r_0^2}{2}+(1+2 B_1 r_0)^2$.
Recall $r_0=\sqrt{\frac{2}{\mu}}$. Thus we obtain  $\rho=\sup_{\bw,z,z'}\ell(\bw,z,z')=1+(1+2B_1 \sqrt{2/\mu})^2$.
\end{proof}

\subsection{Metric Learning}
In supervised metric learning, the distance between two examples $x$ and $x'$ w.r.t $M\in\mathbb{S}_{+}^d$  is defined by $\|x-x'\|_M^2=(x-x')^{\top}M(x-x')$, where $\mathbb{S}_{+}^d$ denotes the cone of all $d \times d$ p.s.d. matrices. 
For every pair of examples with labels $(x,y)$ and $(x',y')$, denote $\mathcal{I}_{yy'}=1$ if $y=y'$, otherwise  $\mathcal{I}_{yy'}=-1$. Using the following logistic loss (e.g. \cite{guillaumin2009you}), the ERM formulation for metric learning can be written as
\begin{eqnarray}\label{mel-1}
\min_{M\in\Omega} \sum_{i=1}^n\sum_{j=1}^n\log\left[1+\exp\left(\mathcal{I}_{y_i y_j}(\|x_i-x_j\|_M^2)\right)\right],
\end{eqnarray}
where $\Omega:=\{M\in\mathbb{S}_{+}^d: \|M\|_F\leq r_0\}$ with $\|\cdot\|_F$ denoting the Frobenius norm of matrix. Its population risk can be expressed as  $\EX\left[\log\left(1+\exp\left(\mathcal{I}_{y y'}(\|x-x'\|_M^2\right)\right)\right].$

For any matrices $A$ and $B$, let $\langle A, B\rangle_{\hbox{tr}} = \hbox{trace}(A^\top B)$. In this case,  the model parameter $\bw = M$ and 
$$\ell(\bw,z,z')=\log\left[1+\exp\left\{ \mathcal{I}_{yy'}\langle \bw, (x-x')(x-x')^\top\rangle_{\hbox{tr}}\right\}\right].$$
By Theorem \ref{thm:convex}, we have the following result. 
\begin{corollary}\label{cor:ML}
For the metric learning problem (\ref{mel-1}), the loss function $\ell(\cdot;z,z')$ is $(4 B_1^4)$-smooth, convex and $(4 B_1^2)$-Lipschitz for every example points $z$ and $z'$.
The projected SGD  with the step sizes $\alpha_t\leq 1/(2 B_1^4)$ has the stability
$$
\epsilon_{stab}(A, T, \ell, D,n)\leq \frac{64B_1^4}{n}\sum_{t=1}^{T-1} \alpha_t,$$
where $T$ is the number of updates.
\end{corollary}
\begin{proof}
We first give one claim which is easy to be verified.
Rewrite $\ell(\bw;z,z')=g_1(g_2(\bw))$, where $g_1$ is $L_1-$Lipschitz, $\beta_1-$smooth and $g_2$ is $L_2-$Lipschitz, $\beta_2-$smooth.
Then, $\ell(\bw;z,z')$ is $(L_1L_2)-$Lipschitz and $(L_1\beta_2+L_2^2\beta_1)-$smooth.

Rewrite $\ell(\bw;z,z')=g_1(g_2(\bw))$, where
\begin{eqnarray*}
\left\{
\begin{array}{ccl}
g_{1}(u)&=&\log\{1+\exp(u)\},  \nonumber\\
u&=& g_{2}(\bw),  \nonumber\\
g_{2}(\bw)&=& \mathcal{I}_{yy'}\langle \bw, (x-x')(x-x')^\top\rangle_{\hbox{tr}}.
\end{array}
\right.
\end{eqnarray*}
We have $g_{1}$ is $1-$Lipschitz, $1/4-$smooth and 
$g_{2}$ is $(4 B_1^2)$-Lipschitz, $0-$smooth as $\nabla g_{2}(\bw)=\mathcal{I}_{yy'}(x-x')(x-x')^\top$.
Thus we have $L= 4 B_1^2$ and $\beta=4 B_1^4$.
Substituting these constants into Theorem \ref{thm:convex} we have proved the Corollary \ref{cor:ML}.
\end{proof}

\subsection{Minimum Error Entropy Principle}
For simplification, we concentrate on a simple linear regression case of the general framework of MEE principle in \citep{hu2016convergence,Hu2013,principe}, i.e.,
\begin{eqnarray}\label{mee-1}
\min_{\|\bw\|\leq r_0} R(\bw):=\mathbb{E}[\ell(\bw,z,z')],
\end{eqnarray}
where  the loss $$\ell(\bw,z,z')=1- \exp\left(-\frac{((y-y')-(x-x')^{\top}\bw)^2}{2h^2}\right)$$ with a scaling parameter $h>0$.  It is obvious that the loss function is non-convex. 

Notice that $\ell(\bw,z,z')$ is negative and bounded with 
$
\sup_{\bw,z,z'}(-\ell(\bw,z,z'))=1.
$
Then we can use Theorem \ref{thm:non-convex} to give a uniform stability of the projected SGD for MEE in the following corollary. 

\begin{corollary}\label{cor:MEE}
For MEE problem (\ref{mee-1}), the loss function $\ell(\cdot;z,z')\in[0,1)$ is  $L$-Lipschitz and $\beta$-smooth for every example points $z$ and $z'$ with the following constants
\begin{eqnarray*}
\left\{
\begin{array}{lcl}
L&=& \frac{4}{h^2}\cdot(B_1^2 r_0+B_1 B_2),   \nonumber\\
\beta&=&\frac{4}{h^2}\cdot B_1^2+\frac{16}{h^4}\cdot(B_1^2 r_0)+B_1 B_2)^2.
\end{array}
\right.
\end{eqnarray*}
The projected SGD with step sizes $\alpha_t\leq\frac{c}{t}$ satisfies the approximate uniform stability with
 \begin{eqnarray*}
\epsilon_{stab}(A^{\text{last}}, T, \ell, D,n)\leq \frac{1+1/(\beta c)}{n-1}\cdot(4cL^2)^{\frac{1}{1+\beta c}}(T-1)^{\frac{\beta c}{1+\beta c}},
\end{eqnarray*}
where $T$ is the number of updates.
\end{corollary}

\begin{proof} 
We now calculate $L$ and $\beta$ of the loss $\ell$ in the MEE problem (\ref{mee-1}).
Rewrite $\ell(\bw;z,z')=g_1(g_2(\bw))$, where
$g_1(u)=1-\exp\{-\frac{u^2}{2h^2}\}$, $u= g_2(\bw)$ and $g_2(\bw)=(x-x')^{\top}\bw-(y-y').$
Assume $g_1$ is $L_1-$Lipschitz, $\beta_1-$smooth and $g_2$ is $L_2-$Lipschitz, $\beta_2-$smooth.
Thus we have
\begin{eqnarray*}
\left\{
\begin{array}{lcllcl}
L_1=\frac{2}{h^2}\cdot(B_1 r_0+B_2), \  L_2=2 B_1,   \nonumber\\
\beta_1=\frac{1}{h^2}+\frac{4}{h^4}\cdot(B_1 r_0+B_2)^2,\  \beta_2=0.
\end{array}
\right.
\end{eqnarray*}
And recalling the simple claim at the beginning of the proof of Corollary \ref{cor:ML},  we have
\begin{eqnarray*}
\left\{
\begin{array}{lcl}
L&=& \frac{4}{h^2}\cdot(B_1^2 r_0+B_1 B_2),   \nonumber\\
\beta&=&\frac{4}{h^2}\cdot B_1^2+\frac{16}{h^4}\cdot(B_1^2 r_0+B_1 B_2)^2.
\end{array}
\right.
\end{eqnarray*}
\end{proof}

\section{Conclusion}\label{sec:conclusion}
In this paper we establish the stability and its trade-off with optimization error of SGD algorithms for pairwise learning.  Stability results of SGD hold true for both convex and non-convex cases.  The trade-off results are established by deriving the lower bound for the minimax statistical error from which lower bounds for the convergence rate of SGD can be obtained for the cases of smooth convex and strongly convex losses. Examples are given to illustrate our main results in specific pairwise learning tasks such as AUC maximization, metric learning and MEE principle.  

There are several directions for future work. Firstly, the stability results we established are not data-dependent. It would be nice to obtain data-dependent bounds related to the curvature of the loss function and the geometry of the training data.  Secondly, the lower bounds for optimization error of SGD   have an undesired bias term in Theorems \ref{thm:sconvex-convergence-lower-bound-1} and  \ref{thm:sconvex-convergence-lower-bound-2}. We do not know how to get rid of this term.  
Thirdly, the stability and generalization bounds here can not explain why SGD iterates converge to a good local minimum for the non-convex case of MEE. It was shown by  \cite{hu2016convergence} that the iterates of SGD for pairwise learning converge to the target function for large enough $h$. However, it remains an open question how to establish similar results for a general scaling parameter $h$. Finally, generalization bounds and stability results are obtained in expectation. It is unclear to us how to derive the bounds with high probability.

\section*{Acknowledgments}
This work was done when Wei Shen was a visiting student at SUNY Albany. The corresponding author is Yiming Ying, whose work is supported by the National Science Foundation (NSF) under Grant No \#1816227.
The work of Xiaoming Yuan is supported by the General Research Fund from the Hong Kong Research Grants Council, 12302318.

\appendix
\section{Proof of Lemma \ref{lem:expansive}}\label{app:A}
Let $\ell_{M_t}(\bw)=\frac{1}{t-1}\sum_{j=1}^{t-1} \ell(\bw,z_{\xi_t},z_{\xi_j})$ wherein $M_t=\{z_{\xi_1},\cdots,z_{\xi_t}\}$.
We can simplify the equation of $G_t$ as $G_t(\bw_{t-1})=\bw_{t-1}-\alpha_{t -1}\nabla_\bw \ell_{M_t}(\bw_{t-1}).$
It is obvious that $\ell_{M_t}(\bw)$ has the same properties of convexity and smoothness with $\ell(\bw;z_{\xi_t},z_{\xi_j})$.
Then we prove the three claims in Lemma \ref{lem:expansive}.

 \begin{enumerate}[label=(\alph*)]
\item 
If $\ell$ is $\beta$-smooth, then $\ell_{M_t}(\bw)$ is also $\beta$-smooth.
By the triangle inequality and the $\beta$-smoothness of $\ell_{M_t}$,
\begin{eqnarray*}
&&\|G_t(\bw')-G_t(\bw)\|\leq\|\bw'-\bw\|+\alpha_{t-1}\|\nabla_\bw \ell_{M_t}(\bw')-\nabla_\bw \ell_{M_t}(\bw)\|\nonumber\\
&&\leq\|\bw'-\bw\|+\alpha_{t-1}\beta\|\bw'-\bw\|
=(1+\alpha_{t-1}\beta)\|\bw'-\bw\|.
\end{eqnarray*}

\item 
We have 
\begin{eqnarray}\label{lem:expansive-eq1}
&&\|G_t(\bw')-G_t(\bw)\|^2=\|(\bw'-\bw)-\alpha_{t -1}(\nabla_\bw \ell_{M_t}(\bw')-\nabla_\bw \ell_{M_t}(\bw))\|^2\nonumber\\
&&=\|\bw'-\bw\|^2+\alpha_{t -1}^2\|\nabla_\bw \ell_{M_t}(\bw')-\nabla_\bw \ell_{M_t}(\bw)\|^2\nonumber\\
&&-2\alpha_{t-1}\langle\nabla_\bw \ell_{M_t}(\bw')-\nabla_\bw \ell_{M_t}(\bw),\bw'-\bw\rangle\nonumber\\
&&\leq\|\bw'-\bw\|^2-\Big(\frac{2\alpha_{t-1}}{\beta}-\alpha_{t -1}^2\Big)\|\nabla_\bw \ell_{M_t}(\bw')-\nabla_\bw \ell_{M_t}(\bw)\|^2\nonumber\\
&&\leq\|\bw'-\bw\|^2,
\end{eqnarray}
wherein the first inequality follows from the $\frac{1}{\beta}$-co-coerciveness of $\nabla_\bw \ell_{M_t}(\cdot)$, namely
\begin{eqnarray*}
\langle\nabla_\bw \ell_{M_t}(\bw')-\nabla_\bw \ell_{M_t}(\bw),\bw'-\bw\rangle \geq \frac{1}{\beta}\|\nabla_\bw \ell_{M_t}(\bw')-\nabla_\bw \ell_{M_t}(\bw)\|^2,
\end{eqnarray*} 
since $\ell_{M_t}$ is both convex and $\beta$-smooth from our assumptions of $\ell$.
The last inequality in (\ref{lem:expansive-eq1}) holds because we assume $\alpha_{t-1} \leq \frac{2}{\beta}$.

\item We have $\phi(\bw)=\ell_{M_t}(\bw)-\frac{\gamma}{2}\|\bw\|^2$ is convex and $(\beta-\gamma)$-smooth, which implies the gradient of $\phi$ is $\left(\frac{1}{\beta-\gamma}\right)$-co-coercive. Thus
 \begin{eqnarray*}
&&\langle\nabla_\bw \ell_{M_t}(\bw')-\nabla_\bw  \ell_{M_t}(\bw), \bw'-\bw\rangle  \geq   \frac{\beta\gamma}{\beta+\gamma}\|\bw'-\bw\|^2 \\ 
 &&  +\frac{1}{\beta+\gamma}\|\nabla_\bw \ell_{M_t}(\bw')-\nabla_\bw \ell_{M_t}(\bw)\|^2.
\end{eqnarray*}
With this inequality in mind we have 
\begin{eqnarray*}
&&\|G_t(\bw')-G_t(\bw)\|^2=\|\bw'-\bw\|^2+\alpha_{t -1}^2\|\nabla_\bw \ell_{M_t}(\bw')-\nabla_\bw \ell_{M_t}(\bw)\|^2\nonumber\\
&&-2\alpha_{t-1}\langle\nabla_\bw \ell_{M_t}(\bw')-\nabla_\bw \ell_{M_t}(\bw),\bw'-\bw\rangle\nonumber\\
&&\leq\left(1-2\frac{\beta\gamma\alpha_{t-1}}{\beta+\gamma}\right)\|\bw'-\bw\|^2-\left(\frac{2\alpha_{t-1}}{\beta+\gamma}-\alpha_{t -1}^2\right)\|\nabla_\bw \ell_{M_t}(\bw')-\nabla_\bw \ell_{M_t}(\bw)\|^2\nonumber\\
&&\leq\left(1-\frac{\beta\gamma\alpha_{t-1}}{\beta+\gamma}\right)^2\|\bw'-\bw\|^2,
\end{eqnarray*}
wherein the last inequality follows from our assumption $\alpha_{t-1}\leq\frac{2}{\beta+\gamma}$
and the inequality $\sqrt{1-x}\leq1-\frac{x}{2}$ which holds for $x\in[0,1]$.
\end{enumerate}
$\hfill \Box$


\section{Proof of Theorem \ref{thm:convex-minimax}}\label{App:Proof1}

As we are considering the worst case over the data distribution family $\D$ and the loss function family $\mathcal{L}_{\text{c}}$, we just need to find some special distributions from $\D$ and a specific loss from $\mathcal{L}_{\text{c}}$ and under these specific cases to derive the desired lower bound. 

Specifically, we consider a particular classification problem. 
Recall the sample space $\Z = \X\times \Y$ where $\X$ is a domain in $\R^d$ and $\Y =\{-1,+1\}.$ Naturally, $\Z$ can be divided into two parts, viz., $\Z_{+}=\X \times \{+1\}$ and $\Z_{-}=\X \times \{-1\}$. Also we divide $\X$ into two disjoint parts, namely, $\X_1$ and $\X_2$.

We first consider  a special distribution $P_1$ on the sample space $\Z$. Denote the marginal distribution of $P_1$ on $\X$ by $P_{1\X}.$   We assume $P_{1\X}(x\in\X_1)=P_{1\X} (x\in\X_2)=\frac{1}{2}$. 
Accordingly, we can write $\Z_{+}=(\X_1 \times \{+1\}) \sqcup (\X_2 \times \{+1\})$ and $\Z_{-}=(\X_1 \times \{-1\}) \sqcup (\X_2 \times \{-1\})$ using $\sqcup$ to denote the disjoint union. Then,  define corresponding conditional probabilities as follows: 
\begin{align*}
    P_{1,y|\X}(y=1|x\in \X_1)=\frac{1}{2}+\frac{\nu}{\sqrt{6n}},\quad 
    P_{1,y|\X}(y=-1|x\in \X_1)=\frac{1}{2}-\frac{\nu}{\sqrt{6n}}, \\
    P_{1,y|\X}(y=1|x\in \X_2)=\frac{1}{2}-\frac{\nu-1}{\sqrt{6n}},\quad 
    P_{1,y|\X}(y=-1|x\in \X_2)=\frac{1}{2}+\frac{\nu-1}{\sqrt{6n}},
\end{align*}
wherein the constant $\nu\in(1,\frac{\sqrt{6}}{2})$ to ensure that the above four  probabilities are all in $(0,1)$.
Using the law of total probability, we have
\begin{align*}
    &P_1(z\in\Z_{+})=\frac{1}{2}\cdot\Bigl(\frac{1}{2}+\frac{\nu}{\sqrt{6n}}+\frac{1}{2}-\frac{\nu-1}{\sqrt{6n}}\Bigr)=\frac{1}{2}+\frac{1}{2\sqrt{6n}},\\
    &P_1(z\in\Z_{-})=\frac{1}{2}\cdot\Bigl(\frac{1}{2}-\frac{\nu}{\sqrt{6n}}+\frac{1}{2}+\frac{\nu-1}{\sqrt{6n}}\Bigr)=\frac{1}{2}-\frac{1}{2\sqrt{6n}}.
\end{align*}

Similarly, we can define another distribution $P_2$ on the same splitting of $\Z$. 
Assume $P_{2\X}(x\in\X_1)=P_{2\X} (x\in\X_2)=\frac{1}{2}$.
Its conditional probabilities are given by 
\begin{align*}
    P_{2,y|\X}(y=1|x\in \X_1)=\frac{1}{2}-\frac{\nu}{\sqrt{6n}},
    P_{2,y|\X}(y=-1|x\in \X_1)=\frac{1}{2}+\frac{\nu}{\sqrt{6n}},\\
    P_{2,y|\X}(y=1|x\in \X_2)=\frac{1}{2}+\frac{\nu-1}{\sqrt{6n}},
    P_{2,y|\X}(y=-1|x\in \X_2)=\frac{1}{2}-\frac{\nu-1}{\sqrt{6n}}. 
\end{align*}Then, we have 
\begin{align*}
    P_2(z\in\Z_{+})=\frac{1}{2}-\frac{1}{2\sqrt{6n}},\quad 
    P_2(z\in\Z_{-})=\frac{1}{2}+\frac{1}{2\sqrt{6n}}.
\end{align*}
Let the sample $S_1$  and $S_2$ are i.i.d. drawn from $P_1$ and $P_2$, respectively.

Next, we define a specific convex and $\beta$-smooth loss function from the loss function family  $\mathcal{L}_{\text{c}}$.
Denote $\bw\in\Omega$ as the parameter of the hypothesis function $h$, where $\Omega$ is the parameter space.
Recall that we have assumed $\Omega$ has a finite diameter  i.e. $|\Omega|<\infty$ and for simplicity, we also assume $\Omega$ is centered by 0 without loss of generality. 
Let $\bw[1]$ be the first coordinate of $\bw$ and denote 
\begin{align*}
f_1(\bw)=\left\{\begin{array}{ll}
                  \frac{\beta}{2} (\bw[1]-r)^2  & \text{for}\ |\bw[1]-r|\leq \frac{r}{2}, \\
                  \frac{\beta r}{2} |\bw[1]-r|-\frac{\beta r^2}{8}  & \text{otherwise};
                \end{array}
                \right. \\
f_2(\bw)=\left\{\begin{array}{ll}
                  \frac{\beta}{2} (\bw[1]+r)^2  & \text{for}\ |\bw[1]+r|\leq \frac{r}{2}, \\
                  \frac{\beta r}{2} |\bw[1]+r|-\frac{\beta r^2}{8}  & \text{otherwise}.
                \end{array}
                \right.
\end{align*}
The pairwise loss function ${\ell}(\bw;z,z'):\Omega \times \Z \times \Z \longrightarrow \R$ in our purpose is defined as
\begin{align*}
 {\ell}(\bw;z,z')=\left\{\begin{array}{ll}
 f_1(\bw) & \text{for}\ z\in \Z_{+}, z'\in \Z_{+},\\
 \frac{1}{2}(f_1(\bw)+f_2(\bw))& \text{for}\ z\in \Z_{+}, z'\in \Z_{-} \ \text{or}\ z\in \Z_{-}, z'\in \Z_{+},\\ 
 f_2(\bw) & \text{for}\ z\in \Z_{-}, z'\in \Z_{-}.
 \end{array}\right.
\end{align*}
It is easy to see that that $\ell$ is  convex and $\beta$-smooth with respect to the first argument. 

Now we consider the excess risks of the above specific loss $\ell$ under these two distributions which is given by 
\begin{align*}
     &{R}_1(\bw)=\mathbb{E}_{(z,z')\sim P_1 \times P_1} [{\ell}(\bw;z,z')]\\
                &=\mathbb{P}(z\in \Z_{+}, z'\in \Z_{+})\cdot f_1(\bw)+\mathbb{P}(z\in \Z_{-}, z'\in \Z_{-})\cdot f_2(\bw)\\
                &+\mathbb{P}(z\in \Z_{+}, z'\in \Z_{-})\cdot \frac{1}{2}(f_1(\bw)+f_2(\bw))+\mathbb{P}(z\in \Z_{-}, z'\in \Z_{+})\cdot\frac{1}{2}(f_1(\bw)+f_2(\bw))\\
                &=P_1(z\in \Z_{+})P_1(z'\in \Z_{+})\cdot f_1(\bw)+P_1(z\in \Z_{-})P_1(z'\in \Z_{-})\cdot f_2(\bw)\\
                &+P_1(z\in \Z_{+})P_1(z'\in \Z_{-})\cdot \frac{1}{2}(f_1(\bw)+f_2(\bw))+P_1(z\in \Z_{-})P_1(z'\in \Z_{+})\cdot\frac{1}{2}(f_1(\bw)+f_2(\bw))\\
                &=\left(\frac{1}{2}+\frac{1}{2\sqrt{6n}}\right)^2\cdot f_1(\bw)+\left(\frac{1}{2}-\frac{1}{2\sqrt{6n}}\right)^2\cdot f_2(\bw)\\
                &+\left(\frac{1}{2}+\frac{1}{2\sqrt{6n}}\right)\left(\frac{1}{2}-\frac{1}{2\sqrt{6n}}\right)\cdot (f_1(\bw)+f_2(\bw))\\
                &=\left(\frac{1}{2}+\frac{1}{2\sqrt{6n}}\right)\cdot f_1(\bw)+\left(\frac{1}{2}-\frac{1}{2\sqrt{6n}}\right)\cdot f_2(\bw).
\end{align*}
Similarly, we have that
\begin{align*}
     {R}_2(\bw)=\mathbb{E}_{(z,z')\sim P_2 \times P_2} [{\ell}(\bw;z,z')]=\left(\frac{1}{2}-\frac{1}{2\sqrt{6n}}\right)\cdot f_1(\bw)+\left(\frac{1}{2}+\frac{1}{2\sqrt{6n}}\right)\cdot f_2(\bw).
\end{align*}
Denote the excess risks as $\Delta{R}_1(\bw):={R}_1(\bw)-\inf_{\bw\in\Omega}{R}_1(\bw)$ and  $\Delta{R}_2(\bw):={R}_2(\bw)-\inf_{\bw\in\Omega}{R}_2(\bw)$.

With the above preparations, we are now in the position to use the Le Cam's method (\cite{le2012asymptotic,Tsybakov,wainwright2019high}) to estimate the minimax statistical error, i.e., $\inf_{\hbw}\max_{i\in\{1,2\}}\mathbb{E}_{S_i\sim P_i^n } [\Delta R_i( \hbw(S_i))].$ 
To this end, we write ${R}_1(\bw)$ in details as 
\begin{align*}
    R_1(\bw)
    =\left\{
    \begin{array}{ll}
      \left(\frac{1}{2}+\frac{1}{2\sqrt{6n}}\right) \frac{\beta r}{2}\left(\frac{3r}{4}-\bw[1]\right) +\left(\frac{1}{2}-\frac{1}{2\sqrt{6n}}\right) \frac{\beta r}{2}\left(-\frac{5r}{4}-\bw[1]\right),  & \text{if}\ \bw[1]\leq  \frac{-3r}{2},\\
      \left(\frac{1}{2}+\frac{1}{2\sqrt{6n}}\right) \frac{\beta r}{2}\left(\frac{3r}{4}-\bw[1]\right) +\left(\frac{1}{2}-\frac{1}{2\sqrt{6n}}\right) \frac{\beta}{2}(r+\bw[1])^2, &   \text{if}\ |\bw[1]+r|\leq \frac{r}{2}, \\
      \left(\frac{1}{2}+\frac{1}{2\sqrt{6n}}\right) \frac{\beta r}{2}\left(\frac{3r}{4}-\bw[1]\right) +\left(\frac{1}{2}-\frac{1}{2\sqrt{6n}}\right) \frac{\beta r}{2}\left(\frac{3r}{4}+\bw[1]\right),  &  \text{if}\ |\bw[1]|\leq \frac{r}{2},  \\
      \left(\frac{1}{2}+\frac{1}{2\sqrt{6n}}\right) \frac{\beta}{2}(\bw[1]-r)^2 +\left(\frac{1}{2}-\frac{1}{2\sqrt{6n}}\right) \frac{\beta r}{2}\left(\frac{3r}{4}+\bw[1]\right),  &   \text{if}\ |\bw[1]-r|\leq \frac{r}{2} \\
      \left(\frac{1}{2}+\frac{1}{2\sqrt{6n}}\right) \frac{\beta r}{2}\left(\bw[1]-\frac{5r}{4}\right) +\left(\frac{1}{2}-\frac{1}{2\sqrt{6n}}\right) \frac{\beta r}{2}\left(\frac{3r}{4}+\bw[1]\right),  &   \text{if}\ \bw[1]>\frac{3r}{2}.       
    \end{array}\right.
\end{align*}
Thus, we have
\begin{align*}
    \nabla_{\bw}{R}_1(\bw)
    =\left\{
    \begin{array}{ll}
      \left(-\frac{\beta r}{2},0,\cdots,0\right)^{\top},  &   \text{if } \bw[1]\leq  \frac{-3r}{2},\\
      \left((\frac{1}{2}-\frac{1}{2\sqrt{6n}})\beta \cdot\bw[1] + (\frac{1}{4}-\frac{3}{4\sqrt{6n}})\beta r,0,\cdots,0\right)^{\top},  &   \text{if } |\bw[1]
      +r|\leq \frac{r}{2}, \\
      \left(-\frac{\beta r}{2\sqrt{6n}},0,\cdots,0\right)^{\top},  &  \text{if } |\bw[1]|\leq \frac{r}{2},  \\
      \left((\frac{1}{2}+\frac{1}{2\sqrt{6n}})\beta \cdot\bw[1] + (-\frac{1}{4}-\frac{3}{4\sqrt{6n}})\beta r,0,\cdots,0\right)^{\top},  &   \text{if } |\bw[1] -r|\leq \frac{r}{2} \\
      \left(\frac{\beta r}{2},0,\cdots,0\right)^{\top},  &   \text{if}\ \bw[1]>\frac{3r}{2}.       
    \end{array}\right.
\end{align*}
Let $\bw^{*}_1$ be (any) one of the minimum points of $R_1(\bw)$, i.e. $R_1(\bw^{*}_1)=\inf_{\bw\in\Omega}R_1(\bw)$.  From the explicit form of $\nabla_{\bw}R_1(\bw)$, it is direct to find that 
$(\frac{1}{2}+\frac{1}{2\sqrt{6n}})\beta \cdot\bw^{*}_1[1] + (-\frac{1}{4}-\frac{3}{4\sqrt{6n}})\beta r=0$.
As a result, we have $\bw^{*}_1[1]=\frac{r}{2}+\frac{r}{1+\sqrt{6n}}:=\delta$.
To be more specific, we can further assume that the other coordinates of $\bw^{*}_1$ except $\bw^{*}_1[1]$ all equal to zero.
Thus ${R}_1(\bw^{*}_1)=\inf_{\bw\in\Omega}{R}_1(\bw)=\frac{3(\sqrt{6n}-1)\beta r^2}{8\sqrt{6n}}$.
Denote $\bw_{1,\text{right}}:=(2\delta,0,\ldots,0).$
So we have for any estimator $\hbw$ s.t. $|\hbw[1]-\bw^{*}_1[1]|\geq \delta$, we have  $\Delta{R}_1(\hbw)={R}_1(\hbw)-{R}_1(\bw^{*}_1)\geq \min\{{R}_1(0),{R}_1(\bw_{1,\text{right}})\}-\frac{3(\sqrt{6n}-1)\beta r^2}{8\sqrt{6n}}=\frac{3\beta r^2}{8}-\frac{3(\sqrt{6n}-1)\beta r^2}{8\sqrt{6n}}=\frac{3\beta r^2}{8\sqrt{6n}}$. 

Similarly, we write ${R}_2(\bw)$ in details as 
\begin{align*}
    {R}_2(\bw)
    =\left\{
    \begin{array}{ll}
      \left(\frac{1}{2}+\frac{1}{2\sqrt{6n}}\right) \frac{\beta r}{2}\left(-\bw[1]-\frac{5r}{8}\right) +\left(\frac{1}{2}-\frac{1}{2\sqrt{6n}}\right) \frac{\beta r}{2}\left(\frac{3r}{4}-\bw[1]\right),  & \text{if}\ \bw[1]\leq  \frac{-3r}{2},\\
      \left(\frac{1}{2}+\frac{1}{2\sqrt{6n}}\right) \frac{\beta}{2}(r+\bw[1])^2 +\left(\frac{1}{2}-\frac{1}{2\sqrt{6n}}\right) \frac{\beta r}{2}\left(\frac{3r}{4}-\bw[1]\right), &   \text{if}\ |\bw[1]+r|\leq \frac{r}{2}, \\
      \left(\frac{1}{2}+\frac{1}{2\sqrt{6n}}\right) \frac{\beta r}{2}\left(\bw[1]+\frac{3r}{4}\right) +\left(\frac{1}{2}-\frac{1}{2\sqrt{6n}}\right) \frac{\beta r}{2}\left(\frac{3r}{4}-\bw[1]\right),  &  \text{if}\ |\bw[1]|\leq \frac{r}{2},  \\
      \left(\frac{1}{2}+\frac{1}{2\sqrt{6n}}\right) \frac{\beta r}{2}\left(\bw[1]+\frac{3r}{4}\right) +\left(\frac{1}{2}-\frac{1}{2\sqrt{6n}}\right) \frac{\beta}{2}(\bw[1]-r)^2,  &   \text{if}\ |\bw[1]-r|\leq \frac{r}{2} \\
      \left(\frac{1}{2}+\frac{1}{2\sqrt{6n}}\right) \frac{\beta r}{2}\left(\bw[1]+\frac{3r}{4}\right) +\left(\frac{1}{2}-\frac{1}{2\sqrt{6n}}\right) \frac{\beta r}{2}\left(\bw[1]-\frac{5r}{4}\right),  &   \text{if}\ \bw[1]>\frac{3r}{2}.       
    \end{array}\right.
\end{align*}
Thus, we have
\begin{align*}
    \nabla_{\bw} R_2(\bw)
    =\left\{
    \begin{array}{ll}
      \left(-\frac{\beta r}{2},0,\cdots,0\right)^{\top},  &   \text{if}\ \bw[1]\leq  \frac{-3r}{2},\\
      \left((\frac{1}{2}+\frac{1}{2\sqrt{6n}})\beta \cdot\bw[1] + (\frac{1}{4}+\frac{3}{4\sqrt{6n}})\beta r,0,\cdots,0\right)^{\top},  &   \text{if}\ |\bw[1]+r|\leq \frac{r}{2}, \\
      \left(\frac{\beta r}{2\sqrt{6n}},0,\cdots,0\right)^{\top},  &   \text{if}\ |\bw[1]|\leq \frac{r}{2},  \\
      \left((\frac{1}{2}-\frac{1}{2\sqrt{6n}})\beta \cdot\bw[1] + (-\frac{1}{4}+\frac{3}{4\sqrt{6n}})\beta r,0,\cdots,0\right)^{\top},  &   \text{if}\ |\bw[1]-r|\leq \frac{r}{2} \\
      \left(\frac{\beta r}{2},0,\cdots,0\right)^{\top},  &   \text{if}\ \bw[1]>\frac{3r}{2}.       
    \end{array}\right.
\end{align*}
Let $\bw^{*}_2$ be (any) one of the minimum points of ${R}_2(\bw)$, i.e. ${R}_2(\bw^{*}_2)=\inf_{\bw\in\Omega}{R}_2(\bw)$.  From the explicit form of $\nabla_{\bw}{R}_2(\bw)$, it is direct to find that 
$(\frac{1}{2}-\frac{1}{2\sqrt{6n}})\beta \cdot\bw^{*}_2[1] + (-\frac{1}{4}+\frac{3}{4\sqrt{6n}})\beta r=0$.
So we have $\bw^{*}_2[1]=-\frac{r}{2}-\frac{r}{1+\sqrt{6n}}=-\delta$.
For simplicity, we can further assume that the other coordinates of $\bw^{*}_2$ except $\bw^{*}_2[1]$ all equal to zero.
Thus ${R}_2(\bw^{*}_2)=\inf_{\bw\in\Omega}{R}_2(\bw)=\frac{3(\sqrt{6n}-1)\beta r^2}{8\sqrt{6n}}$. Let $\bw_{2,\text{left}}=(-2\delta,0,\ldots,0).$
So we have for any $\hbw$ s.t. $|\hbw[1]-\bw^{*}_2[1]|\geq \delta$, we have  $\Delta{R}_2(\hbw)={R}_2(\hbw)-{R}_2(\bw^{*}_2)\geq \min\{{R}_2(0),{R}_2(\bw_{2,\text{left}})\}-\frac{3(\sqrt{6n}-1)\beta r^2}{8\sqrt{6n}}=\frac{3\beta r^2}{8}-\frac{3(\sqrt{6n}-1)\beta r^2}{8\sqrt{6n}}=\frac{3\beta r^2}{8\sqrt{6n}}$. 

Combining the above two situations, we have that
for any output $\hbw$, and $\forall i\in\{1,2\}$, 
if $|\hbw[1]-\bw^{*}_i[1]|\geq \delta$, then   $\Delta{R}_i(\hbw)\geq\frac{3\beta r^2}{8\sqrt{6n}}$.

Then, for any $i=1,2$ there holds
\begin{align*}
    \mathbb{E}_{S_i}[\Delta{R}_i(\hbw(S_i))]  \geq   P_i^n (|\hbw[1]-\bw_i^{*}[1]|\geq \delta)\cdot \frac{3\beta r^2}{8\sqrt{6n}}.
\end{align*}
Consequently, 
\begin{align}\label{eq:reduce2P2}
  \inf_{\hbw}\max_{i\in\{1,2\}}\mathbb{E}_{S_i} [\Delta R_i( \hbw(S_i))] \geq  \frac{3\beta r^2}{8\sqrt{6n}} \inf_{\hbw}\max_{i\in\{1,2\}} P_i^n (|\hbw[1]-\bw_i^{*}[1]|\geq \delta). 
\end{align}
By Le Cam's method (\cite{le2012asymptotic,Tsybakov,wainwright2019high}), when $|\bw_1^{*}[1]-\bw_2^{*}[1]|=2\delta$, we can further reduce the estimation of the lower bound of the right hand side of \eqref{eq:reduce2P2} to a binary hypothesis testing problem:
\begin{align}\label{eq:reduce2H}
   \inf_{\hat{\bw}\in \Omega}\max_{i\in\{1,2\}} P_i^n (|\hat{\bw}[1]-\bw_i^{*}[1]|\geq \delta)\geq  \inf_{\Phi}\max_{i\in\{1,2\}} P_i^n (\Phi(\Z_i^n)\neq i),
\end{align}
where the infimum is taken over all binary testing functions $\Phi:\Z^n \rightarrow \{1,2\}$.
Thus by the standard analysis of Le Cam's method, we can further obtain
\begin{align}\label{eq:reduce2kl}
\inf_{\Phi}\max_{i\in\{1,2\}} P_i^n (\Phi(\Z_i^n)\neq i)\geq \frac{1}{2}\cdot(1-\sqrt{ \KL(P_1^n\|P_2^n)/2}),
\end{align}
where $\KL(P_1^n\|P_2^n)$ is the KL divergence. By the assumption of sampling independence, we have $\KL(P_1^n\|P_2^n)=n\KL(P_1\|P_2)$.
Furthermore, using the formulation of the distributions $P_1$ and $P_2$, we have
$\KL(P_1\|P_2)=\frac{1}{\sqrt{6n}}\log\left(\frac{1+\frac{1}{\sqrt{6n}}}{1-\frac{1}{\sqrt{6n}}}\right)$.
Note that $\log\left(\frac{1+x}{1-x}\right)\leq 3x$ for $x\in[0,0.5]$. Thus $\KL(P_1\|P_2)\leq \frac{3}{6n}=\frac{1}{2n}$.
Plugging the above results into \eqref{eq:reduce2kl} gives
\begin{align}\label{eq:lowerbound1}
\inf_{\Phi}\max_{i\in\{1,2\}} P_i^n (\Phi(\Z_i^n)\neq i)\geq \frac{1}{2}\left(1-\sqrt{\frac{1}{4}}\right)=\frac{1}{4}.
\end{align}

Combining the results \eqref{eq:reduce2P2}, \eqref{eq:reduce2H} and \eqref{eq:lowerbound1}, we have
\begin{align}\label{eq:lowerbound2}
  \inf_{\hbw}\max_{i\in\{1,2\}}\mathbb{E}_{S_i} [\Delta{R}_i( \hbw(S_i))] \geq  \frac{3\beta r^2}{8\sqrt{6n}} \cdot\frac{1}{4}=\frac{3\beta r^2}{32\sqrt{6n}}. 
\end{align}
To ensure both $\bw_1^{*}$ and $\bw_2^{*}$ are included in $\Omega$, it must hold that $\|\bw_1^{*}\|_2=\|\bw_2^{*}\|_2=\delta\leq \frac{|\Omega|}{2}$.
Recall that $\delta=\frac{r}{2}+\frac{r}{1+\sqrt{6n}}<r$. Thus it is sufficient to assume $r\leq \frac{|\Omega|}{2}$. This means that we can take $r$ as large as $\frac{|\Omega|}{2}$.
Take this into account and there exists $\ell$ such that 
\begin{equation}\label{eq:lowerbound4}
\inf_{\hbw}\sup_{\mathcal{D\in\mathcal{D}}}\mathbb{E}_{S\sim D^n} [\Delta R( \hbw(S) )] \geq \frac{3\beta |\Omega|^2}{128\sqrt{6n}}.
\end{equation}
The completes the proof of the theorem.  \hfill $\Box$


\section{Proof of Theorem \ref{thm:sconvex-minimax}}\label{App:Proof2}
We will follow the same procedure as the proof for Theorem \ref{thm:convex-minimax}. Specifically, we first define two  distributions $P_1$ and $P_2$ which are exactly the same as the definitions in the proof of Theorem \ref{thm:convex-minimax}.

Then we define a specific strongly convex and strongly smooth loss function.   Let $\Omega$ be the parameter space.
with  a finite diameter i.e. $|\Omega|<\infty$ and without loss of generality, we also assume $\Omega$ is centered by $0$. Denote 
\begin{align*}
f_1(\bw)= \frac{\beta}{2} (\bw[1]-r)^2 + \frac{\beta}{2}\left(\bw[2]^2+\cdots+\bw[d]^2\right),\\
f_2(\bw)=\frac{\beta}{2} (\bw[1]+r)^2 + \frac{\beta}{2}\left(\bw[2]^2+\cdots+\bw[d]^2\right).
\end{align*}
We define the pairwise loss function  ${\ell}(\bw;z,z'):\Omega \times \Z \times \Z \longrightarrow \R$  as
\begin{align*}
 {\ell}(\bw;z,z')=\left\{\begin{array}{ll}
 f_1(\bw) & \text{for}\ z\in \Z_{+}, z'\in \Z_{+},\\
 \frac{1}{2}(f_1(\bw)+f_2(\bw))& \text{for}\ z\in \Z_{+}, z'\in \Z_{-} \ \text{or}\ z\in \Z_{-}, z'\in \Z_{+},\\ 
 f_2(\bw) & \text{for}\ z\in \Z_{-}, z'\in \Z_{-}.
 \end{array}\right.
\end{align*}
It is easy to see that the above loss function $\ell(\bw;z,z')$ is strongly convex and $\beta$-smooth w.r.t $\bw$.
It is sufficient to show that $f_1 (\bw)$ and $f_2 (\bw)$ are both strongly convex and $\beta$-smooth w.r.t $\bw$.
Firstly the Hessian matrices of both $f_1 (\bw)$ and $f_2 (\bw)$ have eigenvalues lower bounded by $\beta>0$. So both $f_1 (\bw)$ and $f_2 (\bw)$ are strongly convex.
To show they are $\beta$-smooth, we calculate the gradients of $f_1(\bw)$ and $f_2(\bw)$.
We have 
$\nabla f_1(\bw)
=\left(\beta (\bw[1]-r),\beta\cdot\bw[2],\ldots,\beta\cdot\bw[d]\right)^{\top}$
It is easy to check that $\|\nabla f_1(\bw_1)-\nabla f_1(\bw_2)\|\leq \beta \|\bw_1-\bw_2\|$.
Similarly we can show $\nabla f_2(\bw)$ is $\beta$-Lipschitz.

Let distributions $P_1$ and $P_2$ be defined as in the proof of Theorem \ref{thm:convex-minimax}. Then, 
\begin{align*}
     {R}_1(\bw)=\mathbb{E}_{(z,z')\sim P_1 \times P_1} [\ell(\bw;z,z')]
                  =\left(\frac{1}{2}+\frac{1}{2\sqrt{6n}}\right)\cdot f_1(\bw)+\left(\frac{1}{2}-\frac{1}{2\sqrt{6n}}\right)\cdot f_2(\bw).
\end{align*}
We denote the excess risk under the distribution $P_1$ as $\Delta{R}_1(\bw):={R}_1(\bw)-\inf_{\bw\in\Omega}{R}_1(\bw)$.
Similarly, 
\begin{align*}
     {R}_2(\bw)=\mathbb{E}_{(z,z')\sim P_2 \times P_2} [{\ell}(\bw;z,z')]
                 =\left(\frac{1}{2}-\frac{1}{2\sqrt{6n}}\right)\cdot f_1(\bw)+\left(\frac{1}{2}+\frac{1}{2\sqrt{6n}}\right)\cdot f_2(\bw).
\end{align*}
We denote the excess risk under the distribution $P_2$ as $\Delta{R}_2(\bw):={R}_2(\bw)-\inf_{\bw\in\Omega}{R}_2(\bw)$.   Consequently,  
\begin{equation}\label{eq:reduce2P-s}
\inf_{\hbw}\sup_{\ell\in\mathcal{L_{\text{sc}}},D\in\mathcal{D}}\mathbb{E}_{S\sim D^n} [\Delta R( \hbw(S)] \geq \inf_{\hbw }\max_{i\in\{1,2\}}\mathbb{E}_{S_i\sim P_i^n} [\Delta{R}_i( \hbw(S_i))].
\end{equation}
Thus it is sufficient to lower bound the right hand side of \eqref{eq:reduce2P-s} using the Le Cam's method (\cite{le2012asymptotic,Tsybakov,wainwright2019high}). 

To this end, we write ${R}_1(\bw)$  as 
\begin{align*}
    &{R}_1(\bw)=\bigl(\frac{1}{2}+\frac{1}{2\sqrt{6n}}\bigr) \frac{\beta}{2}(r-\bw[1])^2+\bigl(\frac{1}{2}-\frac{1}{2\sqrt{6n}}\bigr) \frac{\beta}{2}(r+\bw[1])^2\nonumber\\
    &+ \frac{\beta}{2}\left(\bw[2]^2+\cdots+\bw[d]^2\right)\nonumber\\
    &=\frac{\beta}{2}\bigl(\bw[1]-\frac{r}{\sqrt{6n}}\bigr)^2+\frac{\beta r^2}{2}\bigl(1-\frac{1}{6n}\bigr)+\frac{\beta}{2}\bigl(\bw[2]^2+\cdots+\bw[d]^2\bigr).
\end{align*}
Let $\bw^{*}_1=\arg\min_{\bw\in \gO} {R}_1(\bw)$. It is easy to see that  $\bw^{*}_1[1]=\frac{r}{\sqrt{6n}}:=\delta$ and $\bw^{*}_1[2]=\cdots=\bw^{*}_1[d]=0$.
Thus ${R}_1(\bw^{*}_1)=\inf_{\bw\in\Omega}{R}_1(\bw)=\frac{\beta r^2}{2}\left(1-\frac{1}{6n}\right)$.
Also, for any $\hbw$ s.t. $|\hbw[1]-\bw^{*}_1[1]|\geq \delta$, we have  $\Delta {R}_1(\hbw)={R}_1(\hbw)-{R}_1(\bw^{*}_1)\geq {R}_1(0)-\frac{\beta r^2}{2}\left(1-\frac{1}{6n}\right)=\frac{\beta r^2}{2}-\frac{\beta r^2}{2}\left(1-\frac{1}{6n}\right)=\frac{\beta r^2}{12n}$. 

Likewise, 
\begin{align*}
    {R}_2(\bw)=
    =\frac{\beta}{2}\left(\bw[1]+\frac{r}{\sqrt{6n}}\right)^2+\frac{\beta r^2}{2}\left(1-\frac{1}{6n}\right)+\frac{\beta}{2}\left(\bw[2]^2+\cdots+\bw[d]^2\right).
\end{align*}
It is easy to see that  $\bw^{*}_2=\arg\min_{\bw\in \gO}{R}_2(\bw)$ is given by  $\bw^{*}_2[1]=-\frac{r}{\sqrt{6n}}=-\delta$ and $\bw^{*}_2[2]=\cdots=\bw^{*}_2[d]=0$.
Thus ${R}_2(\bw^{*}_1)=\inf_{\bw\in\Omega}{R}_2(\bw)=\frac{\beta r^2}{2}\left(1-\frac{1}{6n}\right)$.
For any estimator $\hbw$ such that $|\hbw[1]-\bw^{*}_2[1]|\geq \delta$, we have  $\Delta{R}_2(\hbw)={R}_2(\hbw)-{R}_2(\bw^{*}_2)\geq {R}_2(0)-\frac{\beta r^2}{2}\left(1-\frac{1}{6n}\right)=\frac{\beta r^2}{2}-\frac{\beta r^2}{2}\left(1-\frac{1}{6n}\right)=\frac{\beta r^2}{12n}$.  

Combining the above estimation implies the following:  
for any output $\hbw$, and $\forall i\in\{1,2\}$, 
if $|\hbw[1]-\bw^{*}_i[1]|\geq \delta$, then   $\Delta {R}_i(\hbw)\geq\frac{\beta r^2}{12n}$.
Consequently, for any $i=1,2$, we obtain 
\begin{align*}
    \mathbb{E}_{S_i\sim P_i^n}[\Delta R_i(\hbw(S_i))] \geq      P_i^n (|\hbw[1]-\bw_i^{*}[1]|\geq \delta)\cdot \frac{\beta r^2}{12n},
\end{align*}
which implies that 
\begin{align}\label{eq:reduce2P2-s}
  \inf_{\hbw }\max_{i\in\{1,2\}}\mathbb{E}_{S_i\sim P_i^n} [\Delta {R}_i( \hbw(S_i))] \geq  \frac{\beta r^2}{12n}\cdot \inf_{\hbw}\max_{i\in\{1,2\}} P_i^n (|\hbw[1]-\bw_i^{*}[1]|\geq \delta). 
\end{align}
By exactly the same analysis as \eqref{eq:reduce2H}, \eqref{eq:reduce2kl} and \eqref{eq:lowerbound1} in the proof of Theorem \ref{thm:convex-minimax},  we further have 
\begin{align}\label{eq:lowerbound1-s}
\inf_{\hbw}\max_{i\in\{1,2\}} P_i^n (|\hbw[1]-\bw_i^{*}[1]|\geq \delta)\geq \frac{1}{2}\left(1-\sqrt{\frac{1}{4}}\right)=\frac{1}{4}.
\end{align}

Combining the results \eqref{eq:reduce2P2-s}, \eqref{eq:reduce2P2-s} and \eqref{eq:lowerbound1-s}, we have
\begin{align}\label{eq:lowerbound2-s}
 \inf_{\hbw}\max_{D\in \D}\mathbb{E}_{S\sim D^n} [\Delta {R}( \hbw(S))] & \ge \inf_{\hbw}\max_{i\in\{1,2\}}\mathbb{E}_{S_i\sim P_i^n} [\Delta {R}_i( \hbw(S_i))] \\ &\geq  \frac{\beta r^2}{12n} \cdot\frac{1}{4}=\frac{\beta r^2}{48n}. 
\end{align}
This completes the proof the theorem.  \hfill $\Box$

\bibliographystyle{plain}
\bibliography{bibfile.bib}

\end{document}